\documentclass[final]{siamltex}

\newcommand{\black}[1]{\textcolor{black}{#1}}

\usepackage{graphicx} 
\usepackage{subfigure} 
\usepackage{booktabs}
\usepackage{multirow}
\usepackage{algorithm}
\usepackage{algorithmic,color}

\usepackage{pgfplots}\pgfplotsset{compat=1.4}

\usepackage{pict2e}
\usepackage{fullpage}
\usepackage{subfigure}

\usepackage{graphicx}          %
 \usepackage{algorithm, algorithmic}                              %

\RequirePackage[OT1]{fontenc}
\RequirePackage{amsmath,amssymb}
\numberwithin{equation}{section}

\newcommand{\Sc}[1]{{\mathcal{#1}}}
\newcommand{\R}[1]{{\rm #1}}

\title{\black{Fast methods for denoising matrix completion formulations, with applications to robust seismic data interpolation. }}

\author{Aleksandr Aravkin \thanks{IBM T.J. Watson Research Center,
         Yorktown Heights, NY 10598, USA ({\tt saravkin@us.ibm.com}).}
        \and Rajiv Kumar \thanks{Department of Earth and Ocean Sciences, UBC,
        Vancouver, BC, Canada ({\tt rakumar@eos.ubc.ca} )}
        \and Hassan Mansour \thanks{Mitsubishi Electric Research Laboratories, Cambridge, MA, USA ({\tt mansour@merl.com})}
        \and Ben Recht \thanks{Electrical Engineering and Computer Sciences, UC Berkeley,
        Berkeley, CA, USA ({\tt brecht@berkeley.edu})}
        \and Felix J. Herrmann \thanks{Department of Earth and Ocean Sciences, UBC,   Vancouver, BC, Canada ({\tt fherrmann@eos.ubc.ca})}}

\begin{document} 

\maketitle

\begin{abstract} 
Recent SVD-free matrix factorization formulations have enabled 
rank minimization 
for systems with millions of rows and columns,  
paving  the way for matrix completion
in  extremely large-scale applications, such as seismic data interpolation. 

In this paper, we consider matrix completion formulations
designed to hit a target data-fitting error level provided by the user,
and propose an algorithm \black{called LR-BPDN} 
that is able to exploit factorized formulations
to solve the corresponding optimization problem. 
Since practitioners 
typically have strong prior knowledge about target error level, this 
innovation makes it easy to apply the algorithm in practice, 
leaving only the factor rank to be determined.  

Within the established framework, we propose two extensions that are highly relevant to 
solving practical challenges of data interpolation. 
First, we propose a weighted extension that allows 
known subspace information to improve 
the results of matrix completion formulations.
We show how this weighting can be used in the context of frequency 
continuation, an essential aspect to seismic data interpolation.  
{Second, we propose matrix completion formulations that are robust 
to large measurement errors in the available data.} 

We illustrate the advantages of \black{LR-BPDN} 
on collaborative filtering problem using the \black{ MovieLens 1M, 10M, and Netflix 100M datasets}. 
Then, we use the new method, along with its robust and subspace re-weighted 
 extensions, to obtain high-quality
reconstructions for large scale  seismic 
interpolation problems with real data, 
even in the presence of data contamination.

\end{abstract} 

\section{Introduction}

Sparsity- and rank-regularization have had significant impact in many areas
over the last several decades. 
Sparsity in certain transform domains has been exploited to 
solve underdetermined linear systems with applications to compressed sensing 
\cite{Donoho2006_CS,candes2006nos}, 
natural image denoising/inpainting 
\cite{MCA3-2005,Elad:08,Mansour:ICASSP10},
and seismic image processing \cite{Herrmann2008,neelamani2010esf,Herrmann11TRfcd,Mansour11TRssma}. 
Analogously, low-rank structure has been used to efficiently solve matrix completion problems, 
such as the Netflix Prize problem, along with many other applications, including control, 
system identification, signal processing, and combinatorial
optimization~\cite{Fazel:2002, RechtFazelParrilo2010,Candes2011-JACM}, 
and seismic data interpolation and denoising~\cite{oropeza:V25}.

Regularization formulations for both types of problems 
introduce a regularization functional of the decision 
variable, either by adding an explicit penalty  to the data-fitting 
term
\begin{equation}\tag{QP$_\lambda$}\label{QP}
\min_{x}  \rho(\mathcal{A} (x) - b) + \lambda ||x|| ,
\end{equation}
or by imposing constraints
\begin{equation}\tag{LASSO$_\tau$}\label{LASSO}
\min_{x} \rho(\mathcal{A} (x) - b) \quad \text{s.t. } ||x|| \leq \tau\;.
\end{equation}
In these formulations, $x$ may be either a matrix or a vector, $\|\cdot\|$ 
may be a sparsity or low-rank promoting penalty such as the $\ell_1$ norm $\|\cdot\|_1$ or the matrix nuclear norm $\|\cdot\|_*$, $\mathcal{A}$ may be any 
linear operator that predicts the observed data vector $b$ of size $p\times 1$, 
and $\rho(\cdot)$ is typically taken to be the 2-norm.
 
These approaches require the user to provide
regularization parameters whose values
are typically not known ahead of time, and otherwise may require
fitting or cross-validation procedures. 

The alternate formulation 
\begin{equation}\tag{BPDN$_\eta$}\label{BPDN}
\min_{x} ||x||  \quad \text{s.t. } \rho(\mathcal{A} (x) - b) \leq \eta.
\end{equation}
has been successfully
used for the sparse regularization of large scale systems~\cite{BergFriedlander:2008},  
and proposed for nuclear norm regularization~\cite{BergFriedlander:2011}.
This formulation requires the user to provide an acceptable error bound in
the data fitting domain~\eqref{BPDN}, and is preferable 
for many applications, especially when practitioners
know (or are able to estimate) an approximate data error level. 
\black{We refer to ~\eqref{BPDN},~\eqref{QP} and~\eqref{LASSO} as {\it regularization
formulations}, since all three limit the space of feasible solutions by considering the nuclear norm of the 
decision variable.}

A practical implementation of~\eqref{BPDN}
for large scale matrix completion problems is difficult because of 
the large size of the systems
of interest, which makes SVD-based approaches intractable. 
\black{For example, seismic inverse problems work with 4D data volumes, and
matricization of such data creates structures whose size is a bottleneck for standard low-rank 
interpolation approaches. }
Fortunately, a growing literature on factorization-based
rank optimization approaches has enabled
matrix completion formulations for~\eqref{QP} and~\eqref{LASSO}
approaches for extremely large-scale 
systems that avoids costly SVD computations
~\cite{Srebro2005,MaxNorm:NIPS2010,RechtRe:2011}.  
\black{These formulations are non-convex, and therefore do not have the same
convergence guarantees as convex formulations for low-rank factorization. In addition,
they require an {\it a priori} rank specification, adding a rank constraint to the original problem.   
Nonetheless, factorized formulations 
can be shown to avoid spurious local minima, so that if a local minimum is found, it will correspond 
to the global minimum in the convex formulation, provided the chosen factor rank was high enough.  
In addition, computational methods for factorized formulations are more efficient, mainly because
they can completely avoid SVD (or partial SVD) computations. }
In this paper, we extend the framework
of~\cite{BergFriedlander:2011} to incorporate matrix factorization 
ideas, enabling the~\eqref{BPDN} formulation for 
rank regularization of large scale problems, 
such as seismic data interpolation.

While formulations in~\cite{BergFriedlander:2008,BergFriedlander:2011} choose $\rho$ in~\eqref{BPDN} to be the quadratic penalty,  
recent extensions~\cite{AravkinBurkeFriedlander:2013}
allow more general penalties to be used. In particular, robust convex 
(see e.g.~\cite{Hub}) and nonconvex penalties (see e.g.~\cite{Lange1989, AravkinFHV:2012})
can be used to measure misfit error in the~\eqref{BPDN} formulation. 
\black{We incorporate these extensions into our framework, allowing matrix completion formulations that are robust to data contamination. }

Finally, subspace information can be used to inform the matrix completion problem,
analogously to how partial support information can be used to improve the sparse recovery problem~\cite{Mansour:TIT12}.
This idea is especially important for seismic interpolation, where {\it frequency continuation} is used. 
We show that subspace information can be incorporated into the proposed framework using reweighting, 
\black{and that the resulting approach can improve recovery SNR in a frequency continuation setting. Specifically, 
subspace information obtained at lower frequencies can be incorporated into reweighted formulations 
for recovering data at higher frequencies.} 

\textcolor{black}
{
To summarize, we design factorization-based formulations and algorithms for matrix completion that 
\begin{enumerate}
\item Achieve a specified target misfit level provided by the user (i.e. solve~\eqref{BPDN}).
\item Achieve recovery in spite of severe data contamination using robust cost functions $\rho$ in~\eqref{BPDN}
\item Incorporate subspace information into the inversion using re-weighting. 
\end{enumerate}
}

The paper proceeds as follows. In section~\ref{sec:MathematicalFormulations},
we briefly discuss and compare the formulations~\eqref{QP}, \eqref{LASSO}, and~\eqref{BPDN}. 
We also review the SPG$\ell_1$ algorithm~\cite{BergFriedlander:2008}
to solve~\eqref{BPDN}, along with recent extensions for~\eqref{BPDN} formulations 
developed in~\cite{AravkinBurkeFriedlander:2013}.
In section~\ref{sec:RankOptimization},
we formulate the convex relaxation for the rank optimization problem,
and review SVD-free factorization methods. 
In section~\ref{sec:BurMont}, we extend analysis from~\cite{Burer03localminima}
to characterize the relationship between local minima of rank-optimization problems
and their factorized counterparts \black{in a general setting  that captures all formulations of interest here}.
In section~\ref{sec:NewAlgorithm}, we propose an algorithm that combines matrix factorization
with the approach developed by~\cite{BergFriedlander:2008,BergFriedlander:2011,AravkinBurkeFriedlander:2013}. 
We develop the robust extensions in section~\ref{sec:Robust}, 
and reweighting extensions in section~\ref{sec:reweighting}.
Numerical results for both the Netflix Prize problem and for 
seismic trace interpolation of real data are presented in section~\ref{sec:NumericalResults}.

\section{Regularization formulations}
\label{sec:MathematicalFormulations}

Each of the three formulations~\eqref{QP},~\eqref{LASSO}, and~\eqref{BPDN} 
controls the tradeoff between data fitting and a regularization
functional using a regularization parameter. 
However, there are important differences between them. 

From an optimization perspective, most algorithms solve~\eqref{QP} or~\eqref{LASSO},
together with a continuation strategy to modify $\tau$ or $\lambda$, see e.g., \cite{GPSR:2007,BergFriedlander:2008}.
\black{There are also a variety of methods to determine optimal values of the parameters; 
see e.g.~\cite{Giryes2010} and the references within.} 
However, from a modeling perspective~\eqref{BPDN} has a significant advantage, 
since the $\eta$ parameter 
can be directly interpreted as a {\it noise floor},
or a threshold beyond which noise is commensurate with the data. 
In many applications, such as seismic data interpolation,
scientists have good prior knowledge of the noise floor.
\black{In the absence of such knowledge, one still wants an algorithm that 
returns a reasonable solution given a fixed computational budget, 
and some formulations for solving~\eqref{BPDN} satisfy this requirement. }
 
van den Berg and Friedlander~\cite{BergFriedlander:2008} proposed the SPG$\ell_1$ algorithm
for  optimizing~\eqref{BPDN} that captures the features discussed above. 
Their approach solves~\eqref{BPDN} using a series of inexact solutions to~\eqref{LASSO}. 
The bridge between these problems is provided by the {\it value function}  $v : \mathbb{R}\rightarrow \mathbb{R}$
\begin{equation}\label{value}
v(\tau) = \min_x \rho(\mathcal{A}(x) - b) \quad \text{s.t. } \|x\| \leq \tau\;,
\end{equation}
where the particular choice of $\rho(\cdot) = \|\cdot\|^2$ was made in~\cite{BergFriedlander:2008,BergFriedlander:2011}.
The graph of $v(\tau)$ is often called the {\it Pareto curve}.
The~\eqref{BPDN} problem can be solved by finding the root of $v(\tau) = \eta$ using Newton's method:  
\begin{equation}\label{valueNewton}
\tau^{k+1} = \tau^k - \frac{v(\tau) - \eta}{v'(\tau)}\;,
\end{equation}
and the quantities $v(\tau)$ and $v'(\tau)$ can be approximated by solving~\eqref{LASSO} problems. 
In the context of sparsity optimization,~\eqref{BPDN} and~\eqref{LASSO} 
are known to be equivalent for certain values of parameters $\tau$ and $\eta$.
Recently, these results were extended to a much  
broader class of formulations (see~\cite[Theorem 2.1]{AravkinBurkeFriedlander:2013}). 
\black{Indeed, convexity of $\rho$ is not required for this theorem to hold, and instead {\it activity} of the constraint 
at the solution plays a key role.
The main hypothesis requires that solutions $\overline x$ only exist where the constraint is active, i.e.
$\rho(b - \mathcal A(\overline x)) = \sigma$, and $\|\overline x\| = \tau$.} 

For any $\rho$, $v(\tau)$ is non-increasing, since larger $\tau$ 
allow a bigger feasible set.   
For any convex $\rho$ in~\eqref{value}, 
$v(\tau)$ is convex by inf-projection~\cite[Proposition 2.22]{RTRW}. 
When $\rho$ is also differentiable, it follows from~\cite[Theorem 5.2]{AravkinBurkeFriedlander:2013}
that $v(\tau)$ is differentiable, with derivative given in closed form by  
\begin{equation}\label{value_fcn_der}
v'(\tau) = -\|\mathcal{A}^* \nabla \rho(b - \mathcal{A}\bar x)\|_d\;,
\end{equation}
where $\mathcal{A}^*$ is the adjoint to the operator $\mathcal{A}$, $\|\cdot\|_d$
is the dual norm to $\|\cdot\|$, and $\bar x$ solves~\ref{LASSO}.
For example, when the norm $\|\cdot\|$ in~\eqref{value} is the 1-norm, the dual norm 
is the infinity norm, and~\eqref{value_fcn_der} evaluates to the maximum absolute 
entry of the gradient. In the matrix case, $\|\cdot\|$ is typically taken to be the nuclear norm,
and then $\|\cdot\|_d$ is the spectral norm, so~\eqref{value_fcn_der} evaluates to the maximum singular
value of $\mathcal{A}^* \nabla \rho(r)$.

To design effective optimization methods, one has to be able to 
evaluate $v(\tau)$, and to compute the dual norm $\|\cdot\|_d$.
Evaluating $v(\tau)$ requires solving a sequence of 
optimization problems~\eqref{value}, 
for the sequence of $\tau$ given by~\eqref{valueNewton}. 

\textcolor{black}
{
A key idea that makes the approach of~\cite{BergFriedlander:2008} very useful in practice 
is solving LASSO problems inexactly, with increasing precision as the overarching 
Newton's method proceeds. The net computation is therefore much smaller than what would be 
required if one solved a set of LASSO problems to a pre-specified tolerance. 
}
For large scale systems, the method of choice is typically a first-order method, such 
as spectral projected gradient \textcolor{black}{where after taking a step along the negative gradient of the mismatch function $\rho(\mathcal{A}(x) - b)$, the iterate is projected onto the norm ball $\|\cdot\| \leq \tau$}. Fast projection is therefore a necessary 
requirement for tractable implementation, since it is used 
in every iteration of every subproblem. 

With the inexact strategy, the convergence rate of the Newton iteration~\eqref{valueNewton} 
may depend on the conditioning of the linear operator $\Sc{A}$~\cite[Theorem 3.1]{BergFriedlander:2008}.  
For well-conditioned problems, in practice one can often observe only a few (6-10)~\eqref{LASSO} problems
to find the solution for~\eqref{BPDN} for a given $\eta$.
As the optimization proceeds,~\eqref{LASSO} problems for 
larger $\tau$ warm-start from the solution corresponding to 
the previous $\tau$.

\section{Factorization approach to rank optimization}
\label{sec:RankOptimization}

We now consider~\eqref{BPDN} in the specific context of rank minimization. In this setting, 
$\|\cdot\|$ is taken to be the nuclear norm, where for a matrix $X \in \mathbb{R}^{n\times m}$, $\left\| X\right\|_* = \|\sigma\|_1$, 
where $\sigma$ is the vector of singular values. 
The dual norm in this case is $\|\sigma\|_\infty$, which is relatively easy to find
for very large systems. 

\black{Unfortunately, solving the optimization problem in~\eqref{value} is much more difficult. 
For the large system case, this requires repeatedly projecting onto the set 
$\left\| X\right\|_* \leq \tau$, which which means repeated SVD or partial SVD computations. 
This is not feasible for large systems.}


Factorization-based approaches allow 
matrix completion for extremely large-scale 
systems by avoiding costly SVD computations~\cite{Srebro2005,Lee2010,RechtRe:2011}.  
The main idea is to parametrize the matrix $X$ as a product, 
\begin{equation}\label{product}
X = LR^T\;,
\end{equation}
and to optimize over the factors $L,R$. If $X \in \mathbb{R}^{n\times m}$, then $L \in \mathbb{R}^{n\times k}$, 
and $R \in \mathbb{R}^{m \times k}$. The decision variable therefore has dimension $k(n +m)$, rather than $nm$; giving tremendous savings when $k \ll m,n$.
\textcolor{black}{
The asymptotic computational complexity of factorization approaches is the same as that of partial SVDs, as both methods are dominated by an O(nmk) cost; 
the former having to form $X = LR^T$, and the latter computing partial SVDs, at every iteration. However, in practice the former operation is much simpler 
than the latter, and factorization methods outperform methods based on partial SVDs.  In addition, factorization methods keep an explicit bound on the 
rank of all iterates, which might otherwise oscillate, increasing the computational burden.  }

Aside from the considerable savings in the size of the decision variable, the factorization 
technique gives an additional advantage: it allows the use of factorization identities to 
make the projection problem in~\eqref{LASSO} trivial, entirely avoiding the SVD. 

For the nuclear norm, we have~\cite{Srebro2005}
\begin{equation}\label{nucFact}
\|X\|_* = \inf_{X = LR^T} \frac{1}{2}\left\|\begin{bmatrix}L \\ R \end{bmatrix}\right\|_F^2\;.
\end{equation}

Working with a particular representation $X = LR^T$, therefore, guarantees 
that 
\begin{equation}\label{nucInequality}
\|X\|_* = \|LR^T\|_* \leq  \frac{1}{2}\left\|\begin{bmatrix}L \\ R \end{bmatrix}\right\|_F^2\;.
\end{equation}

The nuclear norm is not the only formulation that can be factorized. \cite{MaxNorm:NIPS2010} have recently introduced the 
max norm, which is closely related to the nuclear norm and has been successfully used for matrix completion. 
%
%
%
%

\section{\black{Local minima correspondence between factorized and convex formulations}}
\label{sec:BurMont}

All of the algorithms we propose for matrix completion are based on the factorization approach 
described above.  Even though the change of variables $X = LR^T$ makes the 
problem nonconvex, \black{it turns out that for a surprisingly general class of problems, this change of 
variables does not introduce any extraneous local minima, and in particular any local minimum
of the factorized (non-convex) problem corresponds to a local (and hence global) minimum of the corresponding un-factorized convex problem. }
This result appeared in~\cite[Proposition 2.3]{Burer03localminima} in the context of \black{semidefinite programming (SDP)}; 
however, it holds in general, as the authors 
point out~\cite[p. 431]{Burer03localminima}. 

Here, we state the result for a broad class of problems, which is
general enough to capture all of our formulations of interest. \black{In particular, 
the continuity of the objective function is the main hypothesis required for this correspondence.
It is worthwhile to emphasize this, since in Section~\ref{sec:Robust}, we consider smooth non-convex
robust misfit penalties for matrix completion, which give impressive results (see figure~\ref{fig:9}). }

For completeness, we provide a proof in the appendix. 

\begin{theorem}[General Factorization Theorem]
\label{thm:genFact}
Consider an optimization problem of the form 
\begin{equation}
\label{generalX}
\begin{aligned}
\min_{Z \succeq 0} &\quad f(Z) \\
\text{s.t.} & \quad g_i(Z) \leq 0 \quad i = 1, \dots, n \\
 &\quad h_j(Z) = 0 \quad j = 1, \dots, m \\
 & \quad \mathrm{rank}(Z) \leq r,
\end{aligned}
\end{equation}
where $Z \in \mathbb{R}^{n\times n}$ is positive semidefinite, and $f, g_i, h_i$ are continuous. 
Using the change of variable $Z = SS^T$, take $S \in \mathbb{R}^{n\times r}$, and consider the problem   
\begin{equation}
\label{generalXfact}
\begin{aligned}
\min_{S} &\quad f(SS^T) \\
\text{s.t.} & \quad g_i(SS^T) \leq 0 \quad i = 1, \dots, n \\
 &\quad h_j(SS^T) = 0 \quad j = 1, \dots, m 
\end{aligned}
\end{equation}
Let $\bar Z = \bar S \bar S^T$, where $\bar Z$ is feasible for~\eqref{generalX}. Then 
$\bar Z$ is a local minimum of~\eqref{generalX} if and only if $\bar S$ is a local minimum of~\eqref{generalXfact}. 
\end{theorem}

\black{At first glance, Theorem~\ref{thm:genFact} seems restrictive to apply to a recovery problem for a generic $X$, 
since it is formulated in terms of a PSD variable $Z$. 
However, we show that all of the formulations of interest can be expressed this way, due to the SDP characterization of the nuclear norm. 
}

It was shown in \cite[Sec. 2]{RechtFazelParrilo2010} that the nuclear norm admits a semi-definite programming (SDP) formulation. 
Given a matrix $X \in \mathbb{R}^{n \times m}$, we can \black{characterize the nuclear norm $\|X\|_*$ in terms of an auxiliary matrix positive 
semidefinite marix} $Z \in\mathbb{R}^{(n+m)\times(n+m)}$
\begin{equation}\label{eq:SDP}
	\begin{array}{l}
	\|X\|_*= \min\limits_{Z \succeq 0 } \frac{1}{2}\mathrm{Tr}(Z)\\
		\textrm{subject to } Z_{1,2} = Z_{2,1}^T = X\;,
		\end{array}
\end{equation}
where $Z_{1,2}$ is the upper right $n \times m$ block of $Z$, and $Z_{2,1}$ is the lower left $m \times n$ block. More precisely, the matrix $Z$ is a symmetric positive semidefinite matrix having the structure
\begin{equation}\label{eq:Z}
	Z = 
	\begin{bmatrix} L\\R \end{bmatrix}\begin{bmatrix} L^T & R^T \end{bmatrix}
	=	
	\left[
	\begin{array}{cc}
		LL^T & X \\
		X^T & RR^T
	\end{array}
	\right],
\end{equation}
where $L$ and $R$ have the same rank as $X$, and $\mathrm{Tr}(Z) = \|L\|_F^2 + \|R\|_F^2$.

\black{Using characterization~\eqref{eq:SDP}-\eqref{eq:Z}, }
we can show that a broad class of formulations of interest in this paper 
are in fact problems in the class characterized by Theorem~\ref{thm:genFact}. 

\begin{corollary}[General Matrix Lasso]
\label{cor:genLasso}
Any optimization problem of the form 
\begin{equation}
\label{genLasso}
\begin{aligned}
\min_{X} & \quad f(X)\\
\text{s.t.} & \quad \|X\|_* \leq \tau \\
& \rank(X) \leq r
\end{aligned}
\end{equation}
\black{where $f$ is continuous} has an equivalent problem in the class of problems~\eqref{generalX} characterized by Theorem~\ref{thm:genFact}.
\end{corollary}
\begin{proof}
\black{Using~\eqref{eq:SDP}, write~\eqref{genLasso} as 
\begin{equation}
\label{genLassoRef}
\begin{aligned}
\min_{Z\geq 0} & \quad f(\mathcal{R}(Z))\\
\text{s.t.} &  \quad \mathrm{Tr}(Z) \leq \tau\\
&\rank(Z) \leq r,
\end{aligned}
\end{equation}
where $\mathcal{R}(Z)$ extracts the upper right $n\times m$ block of $Z$. It is clear that if $\rank(Z) \leq r$, then 
$\rank(X) \leq r$, so every solution feasible for the problem in $Z$ is feasible for the problem in $X$ by~\eqref{eq:SDP}.
On the other hand, we can use the SVD of any matrix $X$ of rank $r$ to write $X = LR^T$, with $\rank(L) = \rank(R) = r$, and then the matrix 
$Z$ in~\eqref{eq:Z}
has rank $r$, contains $X$ in its upper right hand corner, and has as its trace the nuclear norm of $X$. 
In particular, if $X = U\Sigma V^T$, we can use $L = U\sqrt{\Sigma}$, and $R = V\sqrt{\Sigma}$ to get this representation. 
Therefore, every feasible point for the $X$ problem has a corresponding $Z$. 
}
\end{proof}


\section{\black{LR-BPDN} Algorithm}
\label{sec:NewAlgorithm}

The factorized formulations in the previous section have been used 
to design several algorithms for large scale matrix completion and rank 
minimization \cite{MaxNorm:NIPS2010,RechtRe:2011}. However, all of these formulations take the form~\eqref{QP} or~\eqref{LASSO}. \black{The~\eqref{LASSO} formulation enjoys a natural relaxation interpretation, see e.g.~\cite{HerrmannFriedlanderYilmaz:2012}; on the other hand, a lot of work has focused on methods for $\lambda$-selection in~\eqref{QP} formulations, see e.g.~\cite{Giryes2010}.
However,  both formulations require some identification procedure of the parameters $\lambda$ and $\tau$.}
 
Instead, we propose to use the factorized formulations to solve the~\eqref{BPDN} problem by traversing the Pareto curve of the nuclear norm minimization problem. In particular, we integrate the factorization procedure into the SPG$\ell_1$ framework, which allows to find the minimum rank solution by solving a sequence of factorized~\eqref{LASSO} 
subproblems~\eqref{subproblem}. The cost of solving the factorized~\eqref{LASSO}  subproblems is relatively cheap and the resulting algorithm takes advantage of the inexact 
subproblem strategy in~\cite{BergFriedlander:2008}.

For the classic nuclear norm  minimization problem, we define 
\begin{equation}
\label{TrueValue}
v(\tau) = \min_X \|\mathcal{A}(X) - b\|_2^2 \quad \text{s.t. } \|X\|_* \leq \tau\;, 
\end{equation}
and find $v(\tau) = \eta$ using the iteration~\eqref{valueNewton}. 

However, rather than parameterizing our problem with $X$, which requires 
SVD for each projection, we use the factorization formulation, 
exploiting Theorem~\ref{thm:genFact} and Corollary~\ref{cor:genLasso}. 
\black{Specifically,} when evaluating the value function $v(\tau)$, \black{we solve the corresponding factorized formulation}
\begin{equation}\label{subproblem}
\min_{L,R}  \|\mathcal{A}(LR^T) - b\|_2^2 \quad \text{s.t. } \frac{1}{2}\left\|\begin{bmatrix}
L\\R
\end{bmatrix}
\right\|_F^2 \leq \tau\;
\end{equation}
using decision variables $L, R$ with a fixed number of $k$ columns each. 

By Theorem~\ref{thm:genFact} and Corollary~\ref{cor:genLasso}, any local solution to this problem
corresponds to a local solution of the true LASSO problem, \black{subject to a rank constraint $\mathrm{rank}(X)\leq k$. }
\black{We use solution $X = LR^T$ reconstructed from~\eqref{subproblem} to evaluate both $v(\tau)$ and 
its derivative $v'(\tau)$.}
\black{When the rank of $L, R$ is large enough, a local minimum of~\eqref{subproblem} corresponds to 
a local minimum of~\eqref{TrueValue}, and for any convex $\rho$, every local minimum of~\eqref{LASSO} is also a global minimum. 
When the rank of the factors $L$ and $R$ is smaller than the rank of the optimal LASSO solution, 
the algorithm looks for local minima of the rank-constrained LASSO problem. }
\black{Unfortunately, we cannot guarantee that the solutions we find are local minima for~\eqref{subproblem}, rather than
simply stationary points. Nonetheless, this approach works quickly and reliably in practice, as we show in our experiments.}


%

Problem~\eqref{subproblem} is optimized using the spectral projected gradient
algorithm. 
The gradient is easy to compute, and the projection requires rescaling 
all entries of $L,R$ by a single value, which is fast, simple, and parallelizable. 

\black{
To evaluate $v'(\tau)$, we use the formula~\eqref{value_fcn_der} for the Newton step 
corresponding to the original (convex) problem in $X$; this requires  
computing the spectral norm (largest singular value) of 
\[
\mathcal{A}^*(b - \mathcal{A}(\bar L \bar R^T))\;, 
\] 
where $\mathcal{A}^*$  is the adjoint of the linear operator $\mathcal{A}$, while $\bar L$ and $\bar R$ are the solutions to~\eqref{subproblem}.
The largest singular value of the above matrix can be computed  relatively quickly using the power method. 
Again, at every update requiring $v(\tau)$ and $v'(\tau)$, 
we are assuming here that our solution $\bar X = \bar L\bar R^T$ is close to a local minimum of the true
LASSO problem, but we do not have theoretical guarantees of this fact. 
}

%

\black{
\subsection{Initialization}
The factorized LASSO problem~\eqref{subproblem} has a stationary point at $L  = 0, R = 0$. 
This means that in particular, we cannot initialize from this point. Instead, we recommend initializing from
a small random starting point. Another possibility is trying to jump start the algorithm, for example using 
the initialization technique of~\cite[Algorithm 1]{Jain:2013}. 
One can compute the partial SVD of the adjoint of the linear operator $\Sc{A}$ on the observed data:
 \[
 USV^T = \Sc{A}^*b
 \]
 Then $L$ and $R$ are initialized as 
 \[
 L = U\sqrt{S}, \quad R = V\sqrt{S}.
 \]
 This initialization procedure can sometimes result in faster convergence over random initialization. Compared to random initialization, 
 this method has the potential to reduce the runtime of the algorithm by 30-40\% for smaller values of $\eta$, 
 see Table~\ref{table8}.
 The key feature of any initialization procedure is to ensure that the starting value of 
 \[
 \tau_0 = \frac{1}{2}\begin{bmatrix} L_0 \\ R_0 \end{bmatrix}
 \]
is {\it less} than the solution to the root finding problem for~\eqref{BPDN}, $v(\tau) = \eta$. 
}

\begin{table}[ht]
\caption{\black{Summary of the computational time (in seconds) for LR-BPDN, measuring the effect of random versus smart (\cite[Algorithm 1]{Jain:2013}) initialization of $L$ and $R$
 for factor rank $k$ and relative error level $\eta$ for~\eqref{BPDN}. Comparison performed on the 1M MovieLens Dataset. Type of initialization had almost no effect
 on quality of final reconstruction. }
}  \label{table8}
\begin{center}
\black{
\begin{tabular}{cc|c|c|c|c|l}
\cline{3-6}
& & \multicolumn{4}{ c| }{\bf Random initialization} \\ \cline{1-6}
\multicolumn{1}{ |c}{}                        &
\multicolumn{1}{  c|  }{$k$} & 10 & 20 & 30 & 50 \\ \cline{1-6}
\multicolumn{1}{ |c }{\multirow{1}{*}{$\eta$=0.5} } &
\multicolumn{1}{ c| }{} & {3.54} & {5.46} & {4.04} & {8.31} &     \\ \cline{1-6}
\multicolumn{1}{ |c  }{\multirow{1}{*}{$\eta$=0.3} } &
\multicolumn{1}{ c| }{} & 11.90 & 6.14 & 8.42&20.84 &\\ \cline{1-6}
\multicolumn{1}{ |c  }{\multirow{1}{*}{$\eta$=0.2} } &
\multicolumn{1}{ c| }{} & 86.53 & 107.88 & 148.12 & 166.92 \\ \cline{1-6}
\end{tabular}\begin{tabular}{cc|c|c|c|c|l}
\cline{3-6}
& & \multicolumn{4}{ c| }{\bf Smart initialization} \\ \cline{1-6}
\multicolumn{1}{ |c}{}                        &
\multicolumn{1}{  c|  }{$k$} & 10 & 20 & 30 & 50 \\ \cline{1-6}
\multicolumn{1}{ |c }{\multirow{1}{*}{$\eta$=0.5} } &
\multicolumn{1}{ c| }{} & {5.01} & {5.75} & {6.84} & {8.24} &     \\ \cline{1-6}
\multicolumn{1}{ |c  }{\multirow{1}{*}{$\eta$=0.3} } &
\multicolumn{1}{ c| }{} & 11.15 & 18.88 & 12.38&21.02 &\\ \cline{1-6}
\multicolumn{1}{ |c  }{\multirow{1}{*}{$\eta$=0.2} } &
\multicolumn{1}{ c| }{} & 58.78 & 84.29 & 95.07 & 114.21 \\ \cline{1-6}
\end{tabular}
}
\end{center}
\end{table}

\subsection{Increasing $k$ on the fly}

\textcolor{black}{
In factorized formulations, the user must specify a factor rank. 
From a computational perspective, it is better that the rank stay small;
however if it is too small, it may be impossible to solve~\eqref{BPDN}
to a specified error level $\eta$. For some classes of problems, where 
the true rank is known ahead of time (see e.g.~\cite{CPA:CPA21432}), 
one is guaranteed that a solution will exist for a given rank. 
However, if necessary, factor rank can be adjusted
on the fly within our framework. 
}
%

\black{Specifically,} adding columns to $L$ and $R$ can be done on the fly, since 
\[
\begin{bmatrix} L & l \end{bmatrix} \begin{bmatrix} R & r\end{bmatrix}^T = LR^T + lr^T\;.
\]
Moreover, the proposed framework for solving~\eqref{BPDN} is fully compatible with 
this strategy, since the underlying root finding is blind to the factorization representation. Changing $k$ only affects iteration~\eqref{valueNewton}
through $v(\tau)$ and $v'(\tau)$.

\subsection{Computational efficiency}
\label{sec:efficiency}
One way of assessing the cost of \black{LR-BPDN} is to compare the computational cost per iteration of the factorization constrained LASSO subproblems \eqref{subproblem} with that of the nuclear norm constrained LASSO subproblems \eqref{TrueValue}. We first consider the cost for computing the gradient direction. 
\textcolor{black}{A gradient direction for the factor $L$ in the factorized algorithm is given by
$$
	g_L = \mathcal{A}^*\left(\mathcal{A}(LR^T) - b\right)R, 
$$
with $g_R$ taking a similar form. Compare this to a gradient direction for $X$
$$
	g_X = \mathcal{A}^*\left(\mathcal{A}(X) - b\right). 
$$}
\black{First, we consider the cost incurred in working with the residual and decision variables.}
While both methods must compute the action of $\mathcal{A}^*$
on a vector, the factorized formulation must modify factors $L, R$ (at a cost of $O(k(n+m))$ and 
re-form the matrix $X = LR^T$ (at a cost of at most $O(knm)$, for every iteration
and line search evaluation. 
\black{Since $\mathcal{A}$ is a sampling matrix for the applications of interest, 
it is sufficient to form only the entries of $X$ that are sampled by $\mathcal{A}$, 
thus reducing the cost to $O(kp)$, where $p$ is the dimension of the measurement vector $b$. 
The sparser the sampling operator $\Sc{A}$, the greater the savings.  }
\black{Standard approaches update an explicit decision variable $X$, at a cost of $O(nm)$}, 
for every iteration and line search evaluation. If the fraction sampled is smaller than the chosen rank $k$, the 
factorized approach is actually cheaper than the standard method. \black{It is also important to note that 
standard approaches have a memory footprint of $O(mn)$, simply to store the decision variable. In 
contrast, the memory used by factorized approaches are dominated by the size of the observed data. }

We now consider the difference in cost involved in the projection. The main benefit for the factorized formulation 
is that projection is done using the Frobenius norm formulation~\eqref{subproblem}, 
and so the cost is $O(k(n+m))$ for every projection. 
\black{In contrast, state of the art implementations that compute
full or partial SVDs in order to accomplish the projection (see e.g.~\cite{Jain10guaranteedrank,TFOCS}) 
are dominated by the cost of this calculation, which is (in the case of partial k-SVD) 
$O(nmk)$, assuming without loss of generality that $k \leq \min(m,n)$. }

\black{
While the complexity of both standard and factorized iterations is dominated by the term $O(mnk)$, in practice 
forming $X = LR^T$ from two factors with $k$ columns each is still cheaper than computing a k-partial SVD of $X$. 
This essentially explains why factorized methods are faster. While it is possible to obtain further speed up for standard methods
using inexact SVD computations, the best reported improvement is a factor of two or three~\cite{LinWei2010}.}  
\black{To test our approach against a similar approach that uses Lanczos to compute partial SVDs, we modified the projection used by the SPGL1 code 
to use this acceleration. We compare against this accelerated code, as well as against TFOCS~\cite{TFOCS} in section~\ref{sec:NumericalResults} (see Table~\ref{table5}).}

 \textcolor{black}{
Finally, both standard and factorized versions of the algorithm require computing the maximum singular value in order to 
compute $v'(\tau)$.  
The analysis in section \ref{sec:BurMont} shows that if the chosen rank of the factors $L$ and $R$ is larger than or equal to the rank of the global minimizers of the nuclear norm LASSO subproblems, 
then any local  minimizer of the factorized LASSO subproblem corresponds to a global minimizer for the convex nuclear norm LASSO formulation.} 
Consequently, both formulations will have similar of Pareto curve updates, since the derivates are necessarily equal at any 
global minimum whenever $\rho$ is strictly convex\footnote{It is shown in~\cite{AravkinBurkeFriedlander:2013} that for any differentiable convex $\rho$, the dual problem for the residual $r = b - \Sc{A}x$ has a unique solution. Therefore, {\it any} global minimum for~\eqref{LASSO} guarantees a unique residual when $\rho$ is strictly convex, and the claim follows, since the derivative only depends on the residual.}. 

\section{Robust Formulations}
\label{sec:Robust}

Robust statistics~\cite{Hub,Mar} play a crucial role in many real-world applications, 
allowing good solutions to be obtained in spite of data contamination. 
In the linear and nonlinear regression setting, the least-squares problem 
\[
\min_X \|F(X) - b\|_2^2
\]
corresponds to the maximum likelihood estimate of $X$ for the statistical model
\begin{equation}
\label{statModel}
b = F(X) + \epsilon\;,
\end{equation}
where $\epsilon$ is a vector of i.i.d. Gaussian variables. Robust statistical approaches
relax the Gaussian assumption, allowing other (heavier tailed) distributions to be used. 
Maximum likelihood estimates of $X$ under these assumptions are more robust to 
data contamination. Heavy-tailed distributions, in particular the Student's t, yield formulations 
that are \black{more robust to outliers than convex formulations~\cite{Lange1989, AravkinFHV:2012}}.  
\black{This corresponds to the notion of a re-descending {\it influence function}~\cite{Mar}, 
which is simply the derivative of the negative log likelihood.}
 The relationship between densities, penalties, and influence functions is shown in
figure~\ref{GLT-KF}. Assuming that $\epsilon$ has the Student's t density leads 
to the maximum likelihood estimation problem 
\begin{equation}
\label{StudentForm}
\min_{X} \rho(F(x) - b) := \sum_{i} \log(\nu + (F(X)_i - b_i)^2), 
\end{equation}
where $\nu$ is the Student's t degree of freedom. 

A general version of~\eqref{BPDN} was proposed in~\cite{AravkinBurkeFriedlander:2013}, 
allowing different penalty funtionals $\rho$. 
The root-finding procedure of~\cite{BergFriedlander:2008} was extended in~\cite{AravkinBurkeFriedlander:2013} 
to this more general context, and used for root finding for both convex and noncovex $\rho$ (e.g. as in~\eqref{StudentForm}).

\textcolor{black}{The~\eqref{BPDN} formulation for any $\rho$  do not arise directly from a maximum likelihood estimator of~\eqref{statModel}, because
they appear in the constraint. }
However, we can think about penalties $\rho$ as {\it agents} 
who, given an error budget $\eta$, distribute it between elements of the residual. 
The strategy that each agent $\rho$ will use to accomplish this task can be deduced 
from tail features evident in Figure~\ref{GLT-KF}. 
Specifically, the cost of a large residual is prohibitively expensive for the least squares penalty, since its {\it cost} is commensurate with that of a very large number of small residuals. 
For example, $(10\alpha)^2 = 100 \alpha^2$; so a residual of size $10\alpha$
is worth as much as 100 residuals of size $\alpha$ to the least squares penalty. 
Therefore, a least squares penalty will never assign a single residual a relatively large value, since this would 
quickly use up the entire error budget. 
In contrast, $|10\alpha| = 10|\alpha|$, so a residual of size $10\alpha$
is worth only 10 residuals of size $\alpha$ when the $1$-norm penalty is used. 
This penalty is likely to grant a few relatively large errors to certain residuals, if this resulted in a better fit. 
For the penalty in~\eqref{StudentForm}, 
it is easy to see that the cost of a residual of size $10\alpha$ can be worth fewer than $10$ residuals of size $\alpha$, 
and specific computations depend on $\nu$ and actual size of $\alpha$. \black{A nonconvex penalty $\rho$, 
e.g. the one in~\eqref{StudentForm}, allows large residuals, as long as the majority 
of the remaining residuals are fit well. }

From the discussion in the previous paragraph, it is clear that robust penalties are useful as {\it constraints} in~\eqref{BPDN}, 
and can cleverly distribute the allotted error budget $\eta$,  using it for outliers while fitting good data. 
The \black{LR-BPDN} framework proposed in this paper captures the robust extension, allowing robust data interpolation
in situations when some of available data is heavily contaminated. 
To develop this extension, we follow~\cite{AravkinBurkeFriedlander:2013} to define the generalized value function
\begin{equation}
\label{genValue}
v_\rho(\tau) = \min_X \rho(\mathcal{A}(X) - b) \quad \text{s.t. } \|X\|_* \leq \tau\;, 
\end{equation}
and find $v_\rho(\tau) = \eta$ using the iteration~\eqref{valueNewton}. 
As discussed in section~\ref{sec:MathematicalFormulations},
for any {\it convex} smooth penalty $\rho$, 
\begin{equation}
\label{DerivForm}
v_\rho'(\tau) = -\|\mathcal{A}^*\nabla \rho(\bar r)\|_2\;,
\end{equation}
where $\|\cdot\|_2$ is the spectral norm, and $\bar r = \mathcal{A}(\bar X) - b$ 
for optimal solution $\bar X$ that achieves $v_\rho(\tau)$. 
For smooth non-convex $\rho$, e.g.~\eqref{StudentForm}, we still use~\eqref{DerivForm}
in iteration~\eqref{valueNewton}. 

As with standard least squares, we use the factorization formulation to 
avoid SVDs. Note that Theorem~\ref{thm:genFact} and Corollary~\ref{cor:genLasso}
hold for any choice of penalty $\rho$.
When evaluating the value function $v_\rho(\tau)$, we actually solve
\begin{equation}\label{RobustSubproblem}
\min_{L,R}  \rho(\mathcal{A}(LR^T) - b) \quad \text{s.t. } \frac{1}{2}\left\|\begin{bmatrix}
L\\R
\end{bmatrix}
\right\|_F^2 \leq \tau\;. 
\end{equation}
For any smooth penalty $\rho$, including~\eqref{StudentForm}, 
\black{a stationary point for this problem can be found using the projected 
gradient method.}

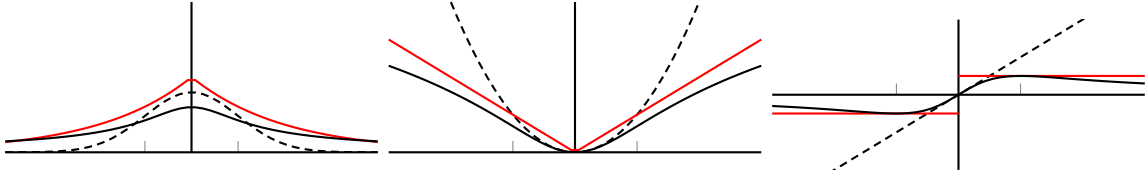
\begin{figure} \label{GLT-KF}
\centering
\begin{tikzpicture}
  \begin{axis}[
    thick,
    width=.3\textwidth, height=2cm,
    xmin=-4,xmax=4,ymin=0,ymax=1,
    no markers,
    samples=50,
    axis lines*=left, 
    axis lines*=middle, 
    scale only axis,
    xtick={-1,1},
    xticklabels={},
    ytick={0},
    ] 
\addplot[domain=-4:+4,densely dashed]{exp(-.5*x^2)/sqrt(2*pi)};
 \addplot[red, domain=-4:+4]{0.5*exp(-.5*abs(x))};
  \addplot[black, domain=-4:+4]{0.3*exp(-.5*ln(1 + x^2))};
  \end{axis}
\end{tikzpicture}
\begin{tikzpicture}
  \begin{axis}[
    thick,
    width=.3\textwidth, height=2cm,
    xmin=-3,xmax=3,ymin=0,ymax=2,
    no markers,
    samples=50,
    axis lines*=left, 
    axis lines*=middle, 
    scale only axis,
    xtick={-1,1},
    xticklabels={},
    ytick={0},
    ] 
\addplot[domain=-3:+3,densely dashed]{.5*x^2};
 \addplot[red, domain=-3:+3]{.5*abs(x)};
  \addplot[black, domain=-3:+3]{.5*ln(1 + x^2)};
  \end{axis}
\end{tikzpicture}
\begin{tikzpicture}
  \begin{axis}[
    thick,
    width=.3\textwidth, height=2cm,
    xmin=-3,xmax=3,ymin=-2,ymax=2,
    no markers,
    samples=50,
    axis lines*=left, 
    axis lines*=middle, 
    scale only axis,
    xtick={-1,1},
    xticklabels={},
    ytick={0},
    ] 
\addplot[domain=-3:3,densely dashed]{x};
 \addplot[red, domain=-3:0]{-.5};
  \addplot[red, domain=0:3]{.5};
\addplot[black, domain=-3:3]{x/(1 + x^2)};
  \end{axis}
\end{tikzpicture}
    \caption{\black{Gaussian (black dashed line), Laplace (red solid line), and Student's t (black solid line); Densities (left plot), Negative Log Likelihoods (center plot), and Influence Functions (right plot).}
   \black{Student's t-density has heavy tails, a non-convex log-likelihood, and re-descending influence function.}}
    
\end{figure}

\section{Reweighting}
\label{sec:reweighting}

Every rank-$k$ solution $\bar{X}$ of~\eqref{BPDN} lives in a lower dimensional subspace of $\mathbb{R}^{n\times m}$ spanned by the $n\times k$ row and $m \times k$ column basis vectors corresponding to the nonzero singular values of $\bar{X}$. In certain situations, it is possible to estimate the row and column subspaces of the matrix $X$ either from prior subspace information or by solving an initial \eqref{BPDN} problem.

\black{In the vector case, it was shown that prior information on the support (nonzero entries) can be incorporated in the $\ell_1$-recovery algorithm by solving the weighted-$\ell_1$ minimization problem.} In this case, the weights are applied such that solutions with large nonzero entries on the support estimate have a lower cost (weighted $\ell_1$ norm) than solutions with large nonzeros outside of the support estimate \cite{Mansour:TIT12}.

\black{In the matrix case, } the support estimate is replaced by estimates of the row and column subspace bases $U_0 \in \mathbb{R}^{n\times k}$ and $V_0 \in \mathbb{R}^{m\times k}$ of the largest $k$ singular values of $X$. Let the matrices $\widetilde{U} \in \mathbb{R}^{n\times k}$ and $\widetilde{V} \in \mathbb{R}^{m\times k}$ be estimates of $U_0$ and $V_0$, respectively. 

The weighted nuclear norm minimization problem can be formulated as follows:
\begin{equation}\tag{wBPDN$_\eta$}\label{wBPDN}
\min_{X} ||QXW||_{*}  \quad \text{s.t. } \rho(\mathcal{A} (X) - b) \leq \eta,
\end{equation}
where $Q = \omega\widetilde{U}\widetilde{U}^T + \widetilde{U}^{\perp}\widetilde{U}^{\perp T}$, $W  = \omega\widetilde{V}\widetilde{V}^T + \widetilde{V}^{\perp}\widetilde{V}^{\perp T}$, and $\omega$ is some constant between zero and one. 
Here, we use the notation $\widetilde{U}^{\perp} \in \mathbb{R}^{n\times n-k}$ to refer to the orthogonal complement of $\widetilde{U}$ in $\mathbb{R}^{n\times n}$, and similarly for $\widetilde{V}^{\perp}$ in $\mathbb{R}^{m\times m}.$
\black{The matrices $Q$ and $W$ are weighted projection matrices of the subspaces spanned by $\widetilde{U}$ and $\widetilde{V}$ and their orthogonal complements. Therefore, minimizing $||QXW||_{*}$ penalizes solutions that live in the orthogonal complement spaces more when $\omega < 1$.}

Note that matrices $Q$ and $W$ are invertible, and hence the reweighed LASSO problem still fits into the class of problems
characterized by Theorem~\ref{thm:genFact}. Specifically, we can write any objective $f(X)$ subject to a reweighted nuclear 
norm constraint as 
\begin{equation}
\label{weightedLasso}
\begin{aligned}
\min &\quad f(Q^{-1}\mathcal{R}(Z)W^{-1})\\
\text{s.t.} & \quad \mathrm{Tr}(Z) \leq \tau\;,
\end{aligned}
\end{equation}
where as in Corollary~\ref{cor:genLasso}, $\mathcal{R}(Z)$ extracts the upper $n\times m$ block of $Z$ (see~\eqref{eq:Z}). 
A factorization similar to~\eqref{subproblem} can then be formulated for the~\eqref{wBPDN} problem in order to optimize over the lower dimensional factors $L \in \mathbb{R}^{n\times k}$ and $R \in \mathbb{R}^{m \times k}$. 

In particular, we can solve a sequence of~\eqref{LASSO} problems
\begin{equation}\label{eq:weightedsubproblem}
\min_{L,R}  \|\mathcal{A}(LR^T) - b\|_2^2 \quad \text{s.t. } \frac{1}{2}
\left\|
\begin{bmatrix}
QL\\
WR
\end{bmatrix}
\right\|_F^2 \leq \tau\;, 
\end{equation}
where $Q$ and $W$ are as defined above. Problem~\eqref{eq:weightedsubproblem} can also be solved using the spectral projected gradient algorithm. However, unlike to the non-weighted formulation, the projection in this case is nontrivial. Fortunately, the structure of the problem allows us to find an efficient formulation for the projection operator.

\subsection{Projection onto the weighted Frobenius norm ball}

The projection of a point $(L, R)$ onto the weighted Frobenius norm ball $\frac{1}{2}\left(\|QL\|_F^2 + \|WR\|_F^2\right) \leq \tau$ is achieved by finding the point $(\widetilde{L}, \widetilde{R})$ that solves
\begin{equation*}\label{eq:weightedProj}
\begin{aligned}
\min\limits_{\hat{L},\hat{R}} \frac{1}{2}\left\|
	\begin{bmatrix}
	\hat{L} - L\\
	\hat{R} - R
	\end{bmatrix}
	\right\|_F^2
	 \quad \text{s.t.}\quad \frac{1}{2}\left\|
	 \begin{bmatrix}
	 Q\hat{L}\\
	 W\hat{R}
	 \end{bmatrix}
	 \right\|_F^2  \leq \tau.
\end{aligned}
\end{equation*}
The solution to the above problem is given by
\begin{equation*}\label{eq:ProjSoln}
\begin{array}{l}
	\widetilde{L} = \left((\mu\omega^2+1)^{-1}\widetilde{U}\widetilde{U}^T + (\mu+1)^{-1}\widetilde{U}^{\perp}\widetilde{U}^{\perp ^T}\right)L\\
	\widetilde{R} = \left((\mu\omega^2+1)^{-1}\widetilde{V}\widetilde{V}^T + (\mu+1)^{-1}\widetilde{V}^{\perp}\widetilde{V}^{\perp ^T}\right)R,
\end{array}
\end{equation*}
where $\mu$ is the Lagrange multiplier that solves $f(\mu) \leq \tau$ with
$f(\mu)$ given by 
\begin{equation}\label{eq:LagrangeMult2}
\begin{aligned}
	 f(\mu) = \frac{1}{2}\text{Tr}\Big[\Big(\frac{\omega^2}{(\mu\omega^2+1)^{2}}\widetilde{U}\widetilde{U}^T + \frac{1}{(\mu+1)^{2}}\widetilde{U}^{\perp}\widetilde{U}^{\perp ^T}\Big)LL^T\\
	 +
	\Big(\frac{\omega^2}{(\mu\omega^2+1)^{2}}\widetilde{V}\widetilde{V}^T + \frac{1}{(\mu+1)^{2}}\widetilde{V}^{\perp}\widetilde{V}^{\perp ^T}\Big)RR^T \Big].
\end{aligned}
\end{equation}
The optimal $\mu$ that solves equation \black{\eqref{eq:LagrangeMult2}} can be found using the Newton iteration 
\begin{equation*}\label{eq:NewtonLambda}
\mu^{(t)} = \mu^{(t-1)} - \frac{f(\mu^{(t-1)}) - \tau}{\nabla f(\mu^{(t-1)})},
\end{equation*}
where $\nabla f(\mu)$ is given by 
\[
\begin{aligned}
 \text{Tr}\Big[\Big(\frac{-2\omega^4}{(\mu\omega^2+1)^{2}}\widetilde{U}\widetilde{U}^T + \frac{-2}{(\mu+1)^{3}}\widetilde{U}^{\perp}\widetilde{U}^{\perp ^T}\Big)LL^T \\
	+
	\Big(\frac{-2\omega^4}{(\mu\omega^2+1)^{3}}\widetilde{V}\widetilde{V}^T + \frac{-2}{(\mu+1)^{3}}\widetilde{V}^{\perp}\widetilde{V}^{\perp ^T}\Big)RR^T \Big].
\end{aligned}
\]

\subsection{Traversing the Pareto curve}

The design of an effective optimization method that solves \eqref{wBPDN} requires 1) evaluating problem \eqref{eq:weightedsubproblem}, and 2) computing the dual of the weighted nuclear norm $\|QXW\|_{*}$. 

\textcolor{black}{We first define a gauge function $\kappa(x)$ as a convex, nonnegative, positively homogeneous function such that $\kappa(0) = 0$. This class of functions includes norms and therefore includes the formulations described in \eqref{wBPDN} and \eqref{eq:weightedsubproblem}. Recall from section \ref{sec:MathematicalFormulations} that taking a Newton step along the Pareto curve of \eqref{wBPDN} requires the computation of the derivative of $v(\tau)$ as in \eqref{value_fcn_der}. Therefore, we also define the polar (or dual) of $\kappa$ as \begin{equation}
	\kappa^{o}(x) = \sup\limits_{w} \{w^Tx \ | \ \kappa(w) \leq 1\}.
\end{equation}
Note that if $\kappa$ is a norm, the polar reduces to the dual norm.}

To compute the dual of the weighted nuclear norm, we follow Theorem 5.1 of \cite{BergFriedlander:2011} which defines the polar (or dual) representation of a weighted gauge function $\kappa(\Phi x)$ as $\kappa^{o}(\Phi^{-1}x)$, where $\Phi$ is an invertible linear operator. The weighted nuclear norm $\|QXW\|_{*}$ is in fact a gauge function with invertible linear weighting matrices $Q$ and $W$. Therefore, the dual norm is given by
\begin{equation*}\label{eq:dualWeightedNucNorm}
	(\|Q(\cdot)W\|_{*})_d(Z) := \|Q^{-1}ZW^{-1}\|_{\infty}.
\end{equation*}

\section{Numerical experiments}
\label{sec:NumericalResults}

We test the performance of \black{LR-BPDN} on two example applications. 
In section~\ref{sec:Netflix}, we consider the Netflix Prize problem, 
which is often solved using rank minimization~\cite{Netflix:2006,Gross:2011,RechtRe:2011}.
\black{Using MovieLens 1M, 10M, and Netflix 100M datasets, we compare and discuss advantages of 
different formulations, compare our solver against state of the art convex~\eqref{BPDN} solver SPG$\ell_1$,
and report timing results. We show that the proposed algorithm is orders of magnitude faster than 
the best convex~\eqref{BPDN} solver. }

In section~\ref{sec:Seismic}, 
we apply the proposed methods and extensions 
to seismic trace interpolation, a key application
in exploration geophysics~\cite{Sacci1998}, 
where rank regularization approaches have recently been 
used successfully~\cite{oropeza:V25}.
In section~\ref{sec:ClassicComparison}, \black{we include an additional comparison of 
LR-BPDN with classic SPG$\ell_1$ as well as with TFOCS~\cite{TFOCS} for small matrix completion 
and seismic data interpolation problems}. 
\black{Then, using real data collected from the Gulf of Suez, we show results for robust
completion in section~\ref{sec:RobustCompletion}, 
and present results for the weighted extension in section~\ref{sec:Reweighting}.} 

\subsection{Collaborative filtering} 
\label{sec:Netflix}
We tested the performance of our algorithm on \textcolor{black}{completing missing entries in
 the MovieLens~(1M),~(10M), and Netflix~(100M) datasets, 
  which contain anonymous ratings of movies made by MovieLens users. }
 The ratings are on an integer scale from 1 to 5. 
 The ratings matrix is not complete, and the goal is to infer the values 
 in the unseen test set. 
 In order to test our algorithm, we further subsampled the available ratings 
 by randomly removing 50$\%$ of the known entries. We then solved 
 the~\eqref{BPDN} formulation to complete the matrix, and compared
 the predicted (P) and actual (A) removed 
 entries in order to assess algorithm performance.
 We report the signal-to-noise ratio (SNR):
  \[ 
  \text{SNR} = 20 \log\left(\frac{\|A\|_F}{\|P-A\|_F}\right)
  \]
 for different values of $\eta$ in the~\eqref{BPDN} formulation.

Since our algorithm requires pre-defining the rank of the factors $L$ and $R$,
we perform the recovery with ranks $k \in \{10,20,30,50\}$. 
Table~\ref{table1} shows the reconstruction SNR for each of the ranks $k$ and for a relative error $\eta \in \{0.5, 0.3, 0.2\}$ 
(the data mismatch  is reduced to a fraction $\eta$ of the initial error).
The last row of table~\ref{table1} shows the recovery for an unconstrained 
low-rank formulation, \black{using the work and software of~\cite{Vand:2013}. 
\black{
This serves as an interesting baseline, since the rank $k$ of the Riemannian manifold in the 
unconstrained formulation functions as a regularizer. 
 It is clear that for small $k$, we get good results without additional functional regularization;
however, as $k$ increases, the quality of the rank $k$ solution decays without further constraints.}
In contrast, we get better results as the rank increases, because we consider a larger model space,
but solve the BPDN formulation each time. 
This observation demonstrates the importance of the nuclear norm regularization, 
especially when the underlying rank of the problem is unknown.
}

\black{
Table 8.2 shows the timing (in seconds) used by all methods to obtain solutions. There are several conclusions that can be readily drawn. 
First, for error-level constrained problems, a tighter error bound requires a higher computational investment by our algorithm, which is consistent
with the original behavior of SPG$\ell_1$~\cite{BergFriedlander:2008}. 
Second, the unconstrained problem is easier to solve (using the Riemmanian manifolds approach of~\cite{Vand:2013}) than a constrained problem 
of the same rank; however, it is interesting to note that as the rank of the representation increases, the unconstrained Riemmanian approach becomes
more expensive than the constrained problem for the levels $\eta$ considered, most likely due to 
second-order methods used by the particular implementation of~\cite{Vand:2013}.  
}

Table~\ref{table4} shows the value of $\|X\|_*$ 
of the reconstructed signal corresponding to the settings in Table~\ref{table1}. 
While the interpretation of the $\eta$ values are straightforward (they are fractions 
of the initial data error),  
it is much more difficult to predict ahead of time which value of $\tau$ one may want 
to use when solving~\eqref{LASSO}. 
This illustrates the {\it modeling} advantage of the~\eqref{BPDN} formulation: 
it requires only the simple parameter $\eta$, \black{which is an estimate of 
the (relative) noise floor}. 
Once $\eta$ is provided, the algorithm 
(not the user) will instantiate~\eqref{LASSO} formulations, and find the right
value $\tau$ that satisfies $v(\tau) = \eta$. 
\black{When no estimate of $\eta$ is available, our algorithm can still be applied to the problem, 
with $\eta = 0$ and a fixed computational budget (see Table~\ref{table6}. }

%


\begin{table}[ht]
\caption{\black{Summary of the recovery results on the MovieLens~(1M) data set for factor rank $k$ and relative error level $\eta$ for~\eqref{BPDN}. 
SNR in dB listed in the left table, and RMSE in the right table. 
The last row in each table gives recovery results for the non-regularized data fitting factorized
 formulation solved with Riemannian optimization (ROPT). Quality {\it degrades} with $k$ due to overfitting 
 for the non-regularized formulation, and improves with $k$ when regularization 
 is used. }}  
        \label{table1}
\begin{center}
\black{
\begin{tabular}{cccccc}
\toprule
& {\bf $k$} & 10 & 20 & 30 & 50\\
{\bf $\eta$} & &&&&\\
\toprule
0.5 && 5.93 & 5.93 & 5.93 & 5.93\\
0.3 && 10.27 & 10.27 & 10.26 & 10.27\\
0.2 && 12.50 & 12.54 & 12.56 & 12.56 \\
\bottomrule
{ROPT}  && 11.16 &8.38 & 6.01 &2.6 \\
\bottomrule
\end{tabular} \quad\quad
\begin{tabular}{cccccc}
\toprule
& {\bf $k$} & 10 & 20 & 30 & 50\\
{\bf $\eta$} & &&&&\\
\toprule
0.5 && 1.89 & 1.89  & 1.89  & 1.89 \\
0.3 && 1.14 &  1.14&  1.15 & 1.14\\
0.2 && 0.88 & 0.88 & 0.88 & 0.88 \\
\bottomrule
{ROPT}  && 1.03 &1.42 & 1.87 & 2.77 \\
\bottomrule
\end{tabular}
}
\end{center}
\end{table}

\begin{table}[ht]
\caption{\black{Summary of the computational timing (in seconds) on the MovieLens~(1M) data set for factor rank $k$ and relative error level $\eta$ for~\eqref{BPDN}. 
The last row gives computational timing for the non-regularized data fitting factorized formulation solved with Riemannian optimization.}}  
        \label{table2}
\begin{center}
\black{
\begin{tabular}{cccccc}
\toprule
& {\bf $k$} & 10 & 20 & 30 & 50\\
{\bf $\eta$} & &&&&\\
\toprule
0.5 && 5.0 & 5.7 & 6.8 & 8.2\\
0.3 && 11.1 & 18.8 & 12.3 & 21.0\\
0.2 && 58.7 & 84.2 & 95.0 & 114.2 \\
\bottomrule
{ROPT} && 14.9 &43.5 & 98.4 &327.3 \\
\bottomrule
\end{tabular}
}
\end{center}
\end{table}


\begin{table}[ht]
\caption{\black{Nuclear-norms of the solutions $X = LR^T$ for results in Table~\ref{table1}, 
corresponding to $\tau$ values in~\eqref{LASSO}.
These values are found automatically via root finding, but are difficult
to guess ahead of time.}}  
        \label{table4}
\begin{center}
\black{
\begin{tabular}{cccccc}
\toprule
& {\bf $k$} & 5 & 10 & 30 & 50\\
{\bf $\eta$} & &&&&\\
\toprule
0.5 && 5.19e3 & 5.2e3 & 5.2e3 & 5.2e3\\
0.3 && 9.75e3 & 9.73e3 & 9.76e3 & 9.74e3\\
0.2 &&  1.96e4& 1.96e4 & 1.93e4 & 1.93e4 \\
\bottomrule
\end{tabular}
}
\end{center}
\end{table}

\black{ Table~\ref{table6} shows a comparison between classic SPG$\ell_1$, accelerated with a Lanczos-based truncated SVD projector, 
against the new solver, on the MovieLens~(10M) dataset, for a fixed budget of 100 iterations. Where the classic solver takes over six hours, 
the proposed method finishes in less than a minute. For a problem of this size, explicit manipulation of $X$ as a full matrix of size 10K by 20K 
is computationally prohibitive. Table~\ref{table7} gives timing and reconstruction quality results for the Netflix~(100M) dataset, where
the full matrix is 18K by 500K when fully constructed. }

\begin{table}[ht]
\caption{ \black{Classic SPGL1 (using Lanczos based truncated SVD) versus LR factorization on the MovieLens~(10M) data set ( $10000 \times 20000$ matrix)
shows results for a fixed iteration budget (100 iterations) to recover 50\% missing entries. SNR, RMSE and
Computational time are shown for $k=5,10,20$. }}
\label{table6}
\begin{center}
\black{
\begin{tabular}{cc|c|c|c|l}
\cline{3-5}
& & \multicolumn{3}{ c| }{\bf MovieLens (10M)} \\ \cline{1-5}
\multicolumn{1}{ |c}{}                        &
\multicolumn{1}{  c|  }{$k$} & 5 & 10 & 20  \\ \cline{1-5}
\multicolumn{1}{ |c  }{\multirow{2}{*}{SPG$\ell_1$} } &
\multicolumn{1}{ |c| }{SNR (dB)} & 11.32 & 11.37 & 11.37&\\ \cline{2-5}
\multicolumn{1}{ |c  }{}                        &
\multicolumn{1}{ |c| }{RMSE} & {1.02} & {1.01} & {1.01} &  \\ \cline{2-5}
\multicolumn{1}{ |c  }{}                        &
\multicolumn{1}{ |c| }{time (sec)} & {\bf 22680} & {\bf 93744} & {\bf 121392} &  \\ \cline{1-5}
\multicolumn{1}{ |c  }{\multirow{2}{*}{LR} } &
\multicolumn{1}{ |c| }{SNR (dB)} & 11.87 & 11.77 & 11.72  \\ \cline{2-5}
\multicolumn{1}{ |c  }{}                        &
\multicolumn{1}{ |c| }{RMSE} & {0.95} & {0.94} & {0.94}&  \\ \cline{2-5}
\multicolumn{1}{ |c  }{}                        &
\multicolumn{1}{ |c| }{ time (sec)} &{\bf 54.3} & {\bf 48.2} & {\bf 47.5} \\ \cline{1-5}
\end{tabular}
}
\end{center}
\end{table}

 \begin{table}[ht]
\caption{ \black{LR method on the Netflix~(100M) data set ( $17770 \times 480189$ matrix)
shows results for 50\% missing entries. SNR, Computational time and RMSE
 are shown for factor rank $k$ and relative error level $\eta$ for~\eqref{BPDN}.}}
\label{table7}
\begin{center}
\black{
\begin{tabular}{cc|c|c|c|l}
\cline{3-5}
& & \multicolumn{3}{ c| }{\bf Netflix (100M) } \\ \cline{1-5}
\multicolumn{1}{ |c}{}                        &
\multicolumn{1}{  c|  }{$k$} & 2 & 4 & 6  \\ \cline{1-5}
\multicolumn{1}{ |c  }{\multirow{2}{*}{$\eta$=0.5} } &
\multicolumn{1}{ |c| }{SNR (dB)} & 7.37 & 7.03 & 7.0&\\ \cline{2-5}
\multicolumn{1}{ |c  }{}                        &
\multicolumn{1}{ |c| }{RMSE} & {1.60} & {1.67} & {1.68} &  \\ \cline{2-5}
\multicolumn{1}{ |c  }{}                        &
\multicolumn{1}{ |c| }{time (sec)} & {\bf 236.5} & {\bf 333.0} & {\bf 335.0} &  \\ \cline{1-5}
\multicolumn{1}{ |c  }{\multirow{2}{*}{$\eta$=0.4} } &
\multicolumn{1}{ |c| }{SNR (dB)} & 8.02 & 7.96 & 7.93  \\ \cline{2-5}
\multicolumn{1}{ |c  }{}                        &
\multicolumn{1}{ |c| }{RMSE} & {1.49} & {1.50} & {1.50} &  \\ \cline{2-5}
\multicolumn{1}{ |c  }{}                        &
\multicolumn{1}{ |c| }{ time (sec)} &{\bf 315.2} & {\bf 388.6} & {\bf 425.0} \\ \cline{1-5}
\multicolumn{1}{ |c  }{\multirow{2}{*}{$\eta$=0.3} } &
\multicolumn{1}{ |c| }{SNR (dB)} & 10.36 & 10.32 & 10.35  \\ \cline{2-5}
\multicolumn{1}{ |c  }{}                        &
\multicolumn{1}{ |c| }{RMSE} & {1.14} & {1.14} & {1.14} &  \\ \cline{2-5}
\multicolumn{1}{ |c  }{}                        &
\multicolumn{1}{ |c| }{ time (sec)} &{\bf 1093.2} & {\bf 853.7} & {\bf 699.7} \\ \cline{1-5}
\end{tabular}
}
\end{center}
\end{table}

\subsection{Seismic missing-trace interpolation}
\label{sec:Seismic}

In exploration seismology, large-scale data sets 
\black{(approaching the order of petabytes for the latest land and wide-azimuth marine acquisitions)} 
must be acquired and processed 
in order to determine the structure of the subsurface. 
In many situations, only a subset of the complete data is acquired due to physical and/or budgetary constraints. 
Recent insights from the field of compressed sensing allow for deliberate subsampling of 
seismic wavefields in order to improve reconstruction quality and reduce acquisition costs \cite{Herrmann2008}. 
The acquired subset of the complete data is often chosen by randomly subsampling a dense regular periodic source or receiver grid. 
Interpolation algorithms are then used to reconstruct the dense regular grid 
in order to perform additional processing on the data such as removal of artifacts, 
improvement of spatial resolution, and key  analysis, such as imaging.

In this section, we apply the new rank-minimization approach, along with 
weighted and robust extensions, 
to the trace-interpolation problem for two different seismic acquisition examples. 
We first describe the structure of the datasets, and then present the transform we use to 
cast the interpolation as a rank-minimization problem. 

The first example is a real data example from the Gulf of Suez.
Seismic data are organized into  {\it seismic lines},
where $N_r$ receivers and $N_s$ sources are collocated in a straight line. 
Sources are deployed sequentially, and receivers record each shot record\footnote{\black{
Data collection performed for several sources taken with increasing or decreasing distance between sources and receivers}. 
} for a period of $N_t$ time samples.
The Gulf of Suez data contains $N_s = 354$ sources, $N_r = 354$ receivers, and $N_t = 1024$ with a sampling interval of 0.004s, leading to a shot duration of 4s and a maximum temporal frequency of 125 Hz. 
Most of the energy of the seismic line is preserved when we restrict the spectrum to the 12-60Hz frequency band. 
Figs.~\ref{fig:1}(a) and (b) illustrate the 12Hz and 60Hz frequency slices in the source-receiver domain, respectively. 
\black{Columns in these frequency slices represent the monochromatic response of the earth to a fixed source and as a function of the receiver coordinate.}
In order to simulate missing traces, we apply a subsampling mask that randomly removes 50\% of the sources, 
resulting in the subsampled frequency slices illustrated in Figs.~\ref{fig:1} (c) and (d). 

\begin{figure*} [ht]
  \begin{center}
    \subfigure[]{\includegraphics[scale=0.25]{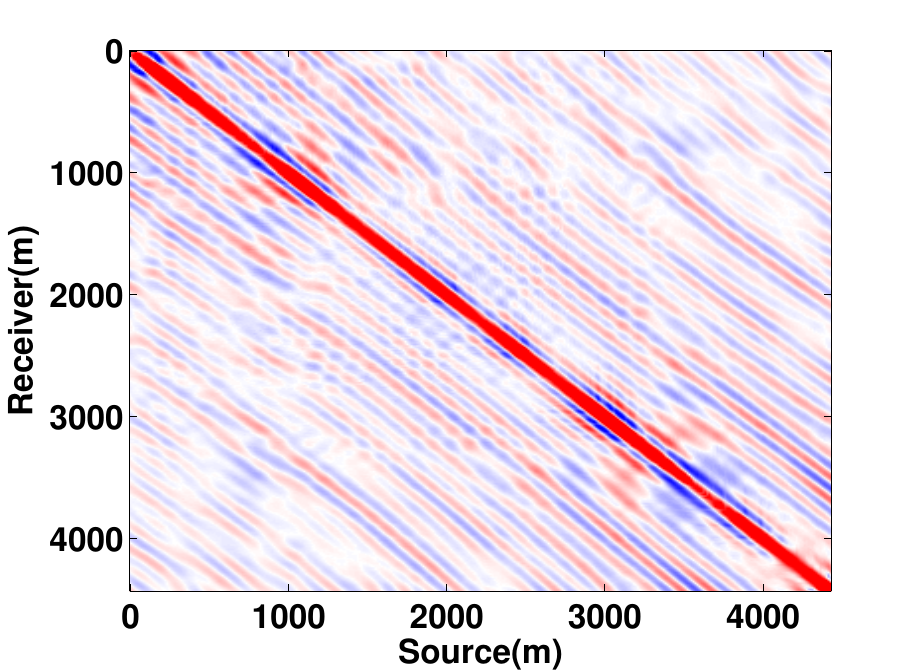}}
    \subfigure[]{\includegraphics[scale=0.25]{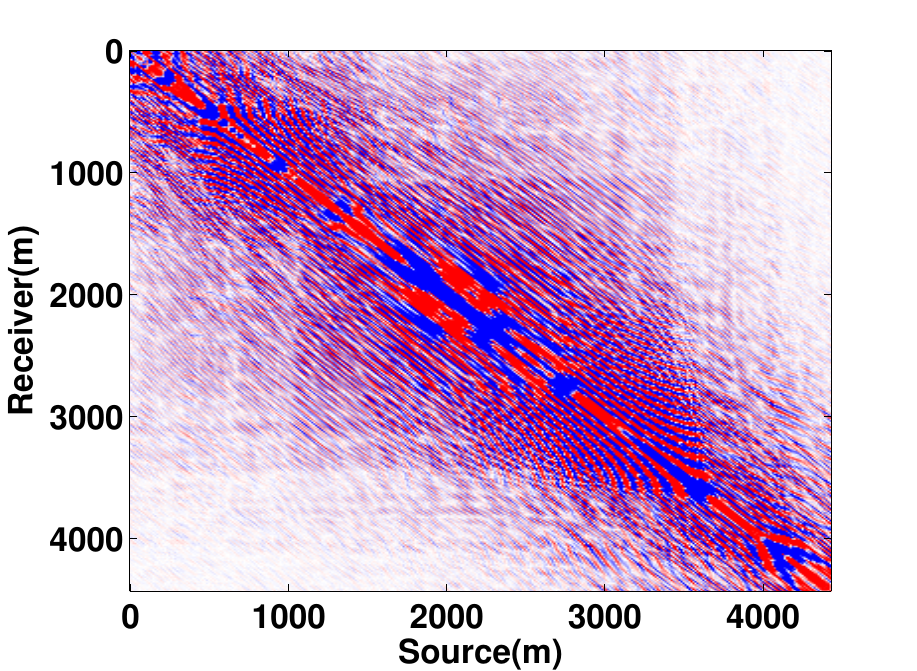}}
    \subfigure[]{\includegraphics[scale=0.25]{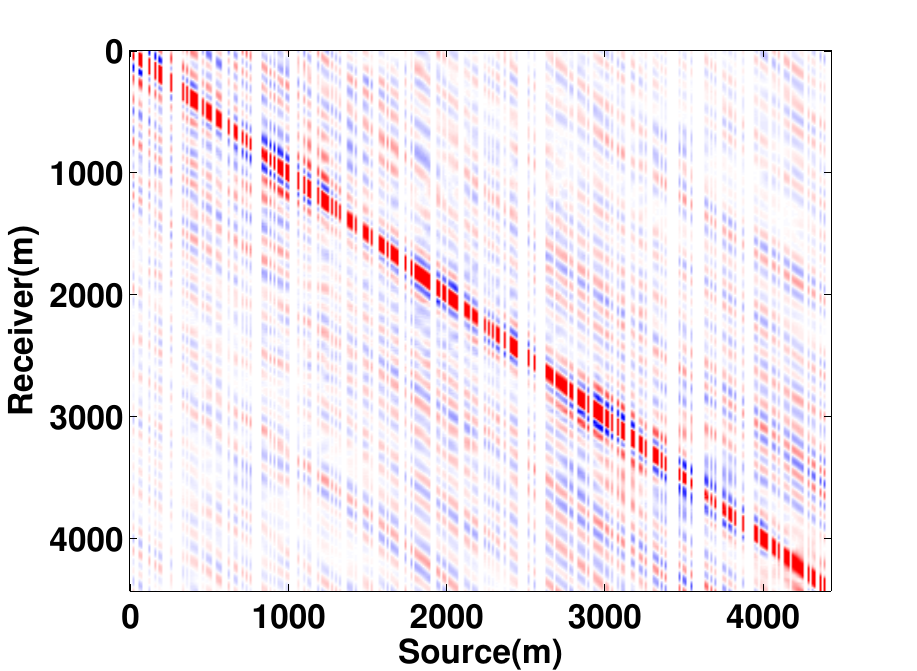}}
    \subfigure[]{\includegraphics[scale=0.25]{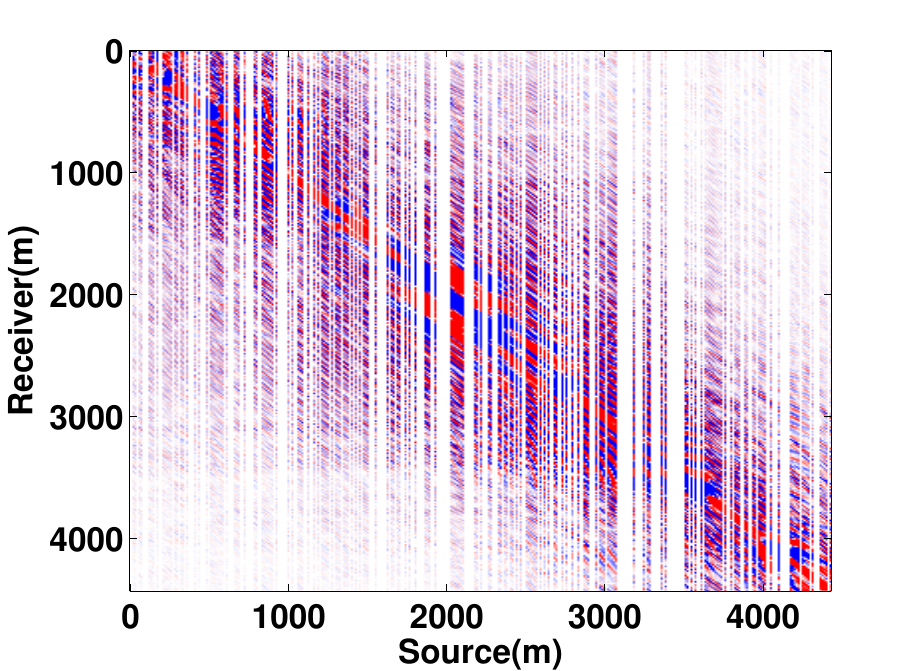}}
     \subfigure[]{\includegraphics[scale=0.25]{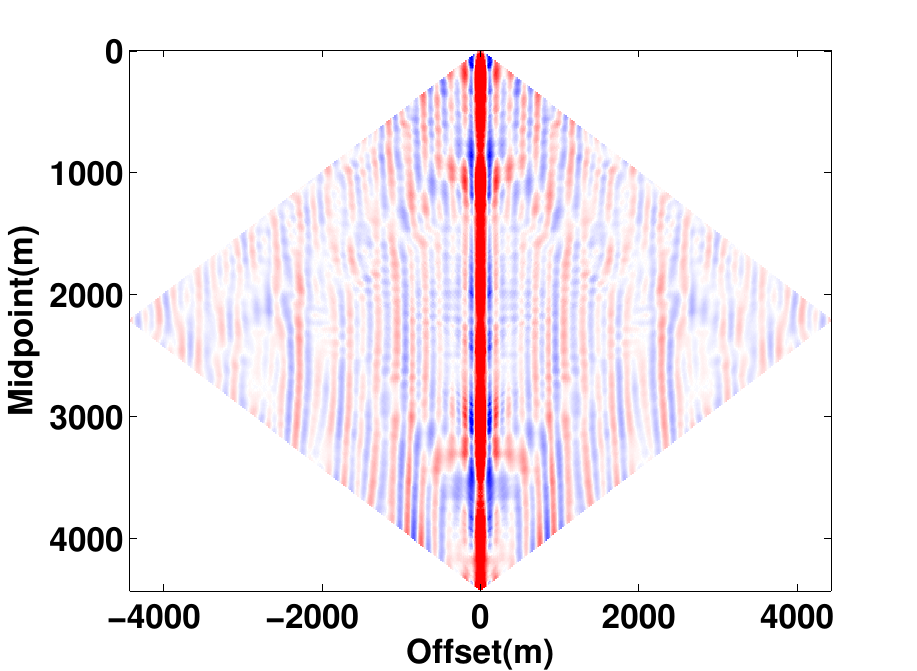}}
    \subfigure[]{\includegraphics[scale=0.25]{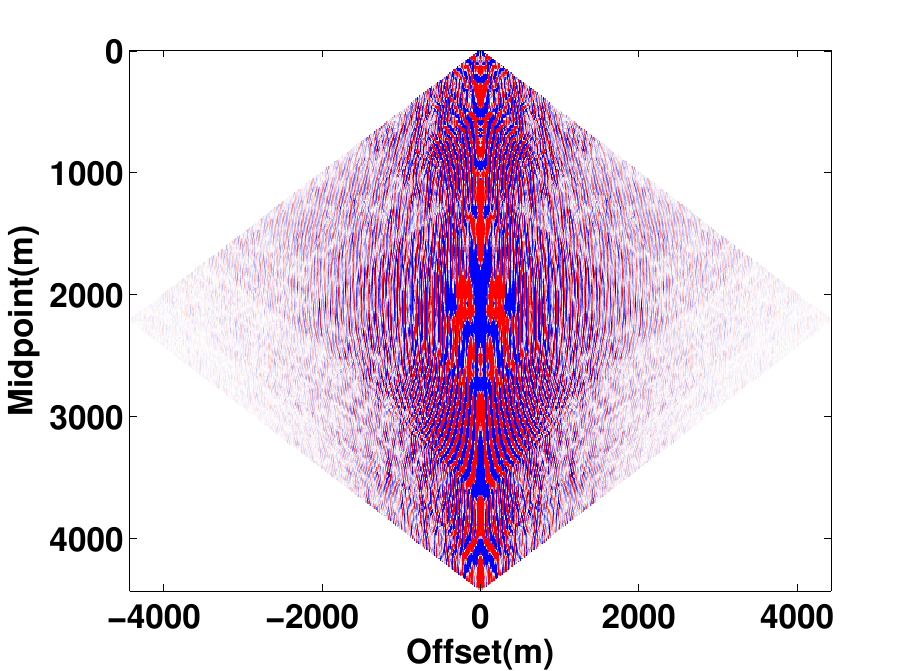}}
    \subfigure[]{\includegraphics[scale=0.25]{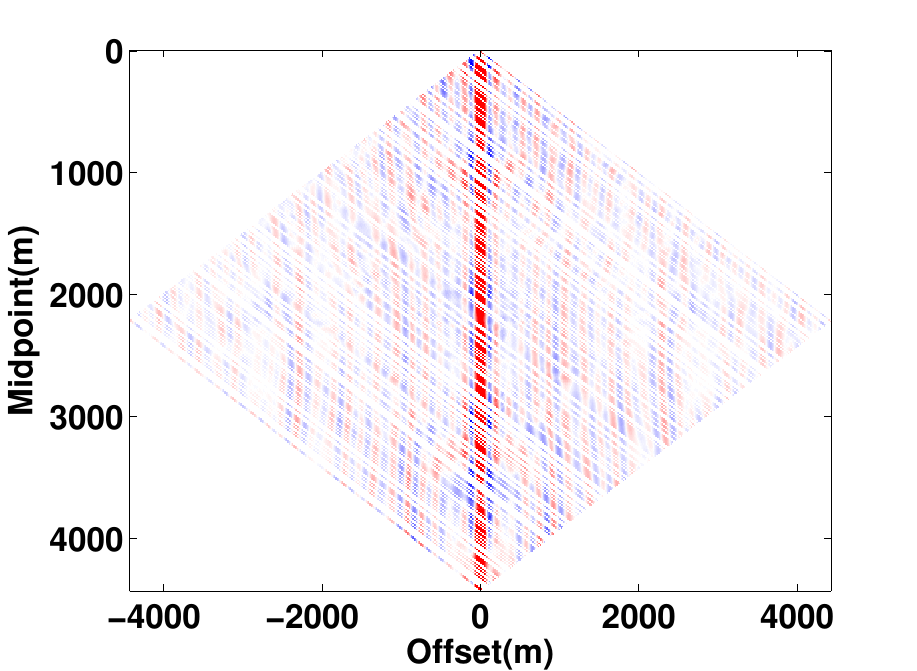}}
    \subfigure[]{\includegraphics[scale=0.25]{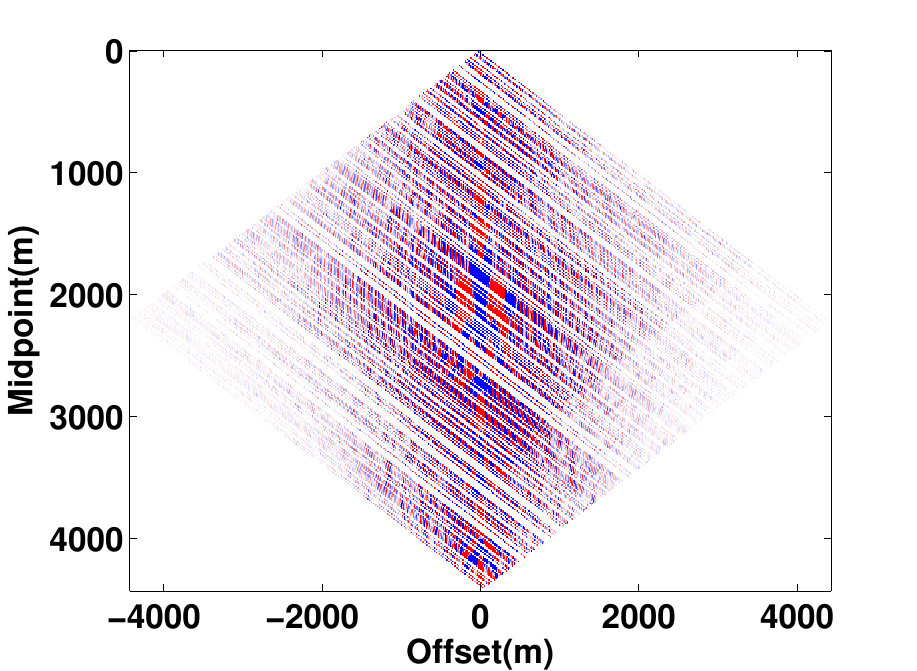}}
  \end{center}
  \caption{Frequency slices of a seismic line from Gulf of Suez with 354 shots, 354 receivers. 
  Full data for
  (a) low frequency at 12 Hz and 
  (b) high frequency at 60 Hz in s-r domain. 
  50\% Subsampled data for
  (c) low frequency at 12 Hz and 
  (d) high frequency at 60 Hz in s-r domain. 
  Full data for  
  (e) low frequency at 12 hz and 
  (f) high frequency at 60 Hz in m-h domain. 
  50\% subsampled data  for 
  (g) low frequency at 12 Hz and 
  (h) high frequency at 60 Hz in m-h domain.}
  \label{fig:1}
\end{figure*}

State of the art trace-interpolation schemes  
transform the data into sparsifying domains, 
for example using the Fourier~\cite{Sacci1998}
and curvelet~\cite{Herrmann2008} transforms. 
The underlying {\it sparse structure} of the data is 
then exploited to recover the missing traces. 
The approach proposed in this paper 
allows us to instead exploit the low-rank {\it matrix structure}
of seismic data, and to design formulations  
that can achieve trace-interpolation using matrix-completion strategies. 

The main challenge in applying rank-minimization for seismic trace-interpolation 
is to find a {\it transform domain} that satisfies the following two properties:
\begin{enumerate}
\item Fully sampled seismic lines have low-rank structure (quickly decaying singular values)
\item  Subsampled seismic 
lines have high rank (slowly decaying singular values).  
\end{enumerate}
When these two properties hold, rank-penalization formulations allow 
the recovery of missing traces. 
%
To achieve these aims, we use the transformation from the 
source-receiver (s-r) domain to the midpoint-offset (m-h). 
The conversion from (s-r) domain to (m-h) domain is a coordinate transformation, 
with the midpoint is defined by m = $\frac{1}{2}$(s+r) 
and the half-offset is defined by h = $\frac{1}{2}$(s-r).\footnote{
\black{In mathematical terms, the transformation from (s-r) domain to (m-h) domain represents a tight frame.}}
This transformation is illustrated by transforming the 12Hz and 60Hz source-receiver domain frequency slices in Figs.~\ref{fig:1}(a) and (b) to the midpoint-offset domain frequency slices in Figs.~\ref{fig:1}(e) and (f). The corresponding subsampled frequency slices in the midpoint-offset domain are shown in Figs.~\ref{fig:1}(g) and (h). 

To show that the midpoint-offset transformation achieves aims 1 and 2 above, 
we plot the decay of the singular values of both the 12Hz and 60Hz frequency slices 
in the source-receiver domain and in the midpoint-offset domain in Figs.~\ref{fig:2} (a) and (c). Notice that the singular values of both frequency slices decay faster in the midpoint-offset domain, and that the singular value decay is slower 
for subsampled data in Figs.~\ref{fig:2} (b) and (d).

\begin{figure}
  \begin{center}
    \subfigure[]{\includegraphics[scale=0.35]{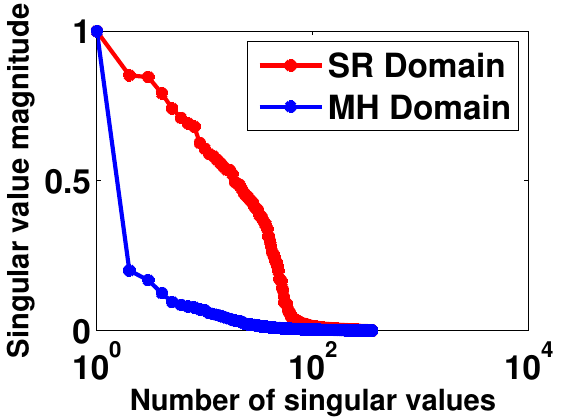}}
    \subfigure[]{\includegraphics[scale=0.35]{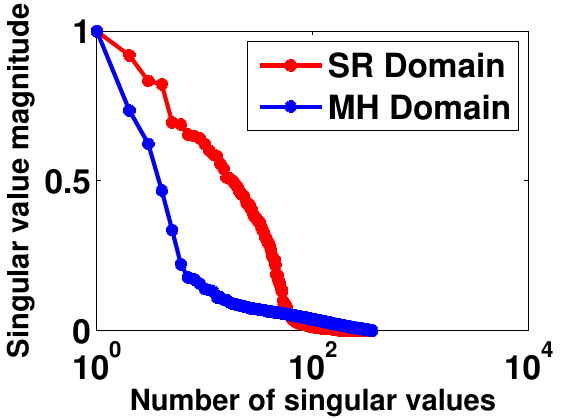}}
    \subfigure[]{\includegraphics[scale=0.35]{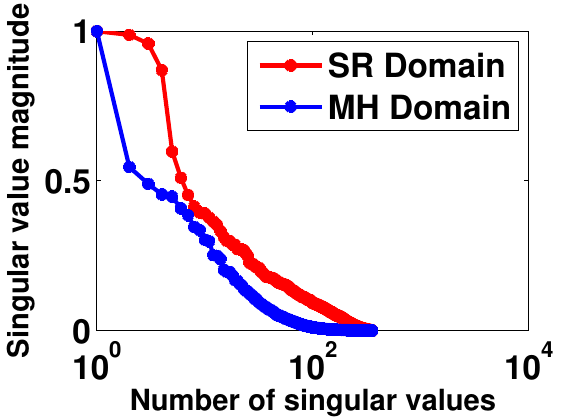}}
    \subfigure[]{\includegraphics[scale=0.35]{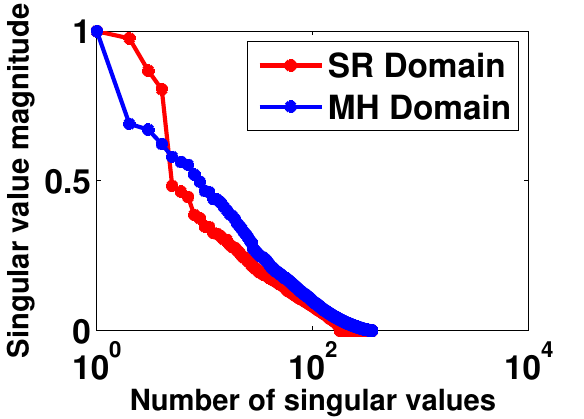}}
  \end{center}
  \caption{Singular value decay of fully sampled 
  (a) low frequency slice at 12 Hz and 
  (c) high frequency slice at 60 Hz in (s-r) and (m-h) domains. 
  Singular value decay of $50\%$ subsampled 
  (b)  low frequency slice at 12 Hz and 
  (d) high frequency data at 60 Hz in (s-r) and (m-h) domains. 
  Notice that for both high and low frequencies, decay of singular values is faster in the fully sampled (m-h) domain 
  than in the fully sampled (s-r) domain,  
  and that subsampling does not significantly change the decay of singular value in (s-r) domain,
  while it destroys fast decay of singular values in (m-h) domain.} 
  \label{fig:2}
\end{figure}

Let $X$ denote the data matrix in the midpoint-offset domain and let $R$ be the subsampling operator that maps Figs.~\ref{fig:1} (e) and (f) to Figs.~\ref{fig:1}(g) and (h). Denote by $\mathcal{S}$ the transformation operator from the source-receiver domain to the midpoint-offset domain. The resulting measurement operator in the midpoint-offset domain is then given by
\(
{{ \mathcal{A} }= R\mathcal{S}^H}\;.
\)

We formulate and solve the matrix completion problem~\eqref{BPDN} 
 to recover a seismic line from Gulf of Suez in the
(m-h) domain. \black{We first performed the interpolation for the frequency slices at 12Hz and 60Hz to get a good approximation of the lower and higher limit of the rank value. Then, we work with all the monochromatic frequency slices
and adjust the rank within the limit while going from low- to high-frequency slices.} 
We use 300 iterations of SPG$\ell_1$ for
all frequency slices. 
Figures~\ref{fig:3}(a) and (b) show the recovery and error plot for the low frequency slice at 12 Hz, respectively. Figures~\ref{fig:3}(c) and (d) show the recovery and error plot for the high frequency slice at 60 Hz, respectively.  
Figures~\ref{fig:4} shows a common-shot gather section after missing-trace interpolation from Gulf of Suez data set. 
\black{We can clearly see that we are able to recapture most of the missing traces in the data (Figures~\ref{fig:4}c), also evident from residual plot (Figures~\ref{fig:4}d).}

In the second acquisition example, we implement the proposed formulation on the 5D synthetic seismic data (2 source dimension, 2 receiver dimension, 1 temporal dimension) provided by BG Group. \black{We extract a frequency slice at 12.3Hz to perform the missing-trace interpolation, where the size of the to-be recovered matrix is $400\times400$ receivers spaced by 25m and $68\times68$ sources spaced by 150m.
Due to the low spatial frequency content of the data at 12.3 Hz, we further subsample the data in receiver coordinates by a factor of two to speed up the computation.} We apply sub-sampling masks that randomly remove 75\% and $50\%$ of the shots. In case of 4D, 
we have two choices of matricization ~\cite{dasilva2013htuck,Demanet2006td}, as shown in Figures~\ref{fig:5}(a,b), where we can either place the (Receiver x, Receiver y) dimensions in the rows and (Source x, Source y) dimensions in the columns, or (Receiver y, Source y) dimensions in the rows and (Receiver x, Source x) dimensions in the columns. We observed the faster decay of singular value decay as shown in Figure \ref{fig:13}(a,b), 
for each of these strategies.  
We therefore selected the transform domain to be the permutation of source and receivers coordinates, where matricization of each 4D monochromatic frequency slices is done using (Source x, Receiver x) and (Source y, Receiver y) coordinates. We use rank 200 for the interpolation, and run the solver for a maximum of 1000 iterations. 
\black{The results after interpolation are shown in Figures \ref{fig:6} and~\ref{fig:7} for 75\% and 50\% missing data, respectively. 
We can see that when 75\% of data is missing, we start losing coherent energy (Figures \ref{fig:6}c). 
With  50\% missing data, we capture most of the coherent energy (Figures \ref{fig:7}c). 
We also have higher SNR values for recovery in case of 50\% compared to 75\% missing data}.

To illustrate the importance of the nuclear-norm regularization, we solved the interpolation problem using a simple least-squares formulation on the same seismic data set from Gulf of Suez. The least squares problem was solved using the $L$, $R$ factorization structure, thereby implicitly enforcing a rank on the recovered estimate (i.e, formulation~\eqref{subproblem} was optimized without the $\tau$-constraint). The problem was then solved with the factors $L$ and $R$ having a rank $k\in \{5, 10, 20, 30, 40, 50, 80, 100\}$. The reconstruction SNRs comparing the recovery for the regularized and non-regularized formulations are shown in Fig.~\ref{fig:8}.  The figure shows that the performance
of the non-regularized approach decays with rank, due to overfitting.  
The regularized approach, in contrast, 
obtains better recovery as the factor rank increases.

\subsubsection{Comparison with classical nuclear-norm formulation}
\label{sec:ClassicComparison}
To illustrate the advantage of proposed matrix-factorization formulation (which we refer to as LR below)
over classical nuclear-norm formulation, we compare the reconstruction error and computation time with the existing techniques. 
The most natural baseline is the SPG$\ell_1$ algorithm~\cite{BergFriedlander:2011} applied to the classic nuclear norm~\eqref{BPDN}
formulation, where the decision variable is $X$, the penalty function $\rho$ is the 2-norm, 
and the projection is done using the SVD.  
\black{This example tests the classic~\eqref{BPDN} formulation against the LR extension proposed in this paper. 
The second comparison is with the TFOCS\cite{TFOCS}, which is a library of first-order methods for a variety of 
problems with explicit examples written by the authors for the~\eqref{BPDN} formulation. The TFOCS approach to~\eqref{BPDN} 
relies on a proximal function for the nuclear norm, which, similar to projection, requires computing SVDs or partial SVDs. }

The comparisons are done using three different data sets. In the first example, we interpolated missing traces of a monochromatic slice (of size $354 \times 354)$,
extracted from Gulf of Suez data set. We subsampled the frequency slice by randomly removing the $50\%$ of shots and performed the missing-trace interpolation in the midpoint-offset (m-h) domain. We compares the SNR, computation time and iterations for a fixed set of $\eta$. 
The rank of the factors was set to $28$. 
The seismic example in table \ref{table5} shows the results.  Both the classic SPG$\ell_1$ algorithm and LR are faster 
than TFOCS. In the quality of recovery, both SPG$\ell_1$ and LR have better SNR than TFOCS. \black{In this case LR is faster than SPG$\ell_1$ by a factor of 15
(see Table~\ref{table5}). }


In the second example, we generated a rank 10 matrix of size $100 \times 100$. 
We subsampled the matrix by randomly removing $50\%$ of the data entries. The synthetic low-rank example in table \ref{table5} shows the comparison of SNR and computational time. \black{The rank of the factors was set to be the true rank of the original data matrix for this experiment.}
\black{
The LR formulation proposed in this paper is faster than classic SPG$\ell_1$, and both 
are faster than TFOCS. As the error threshold tightens, TFOCS requires a large number of iterations to converge.}
\black{For a small problem size, LR and the classic SPG$\ell_1$ perform comparably}.
\black{ When operating on decision variables with the correct rank LR gave uniformly better SNR results than classic SPG$\ell_1$ and TFOCS,
 and the improvement was significant for lower error thresholds.}\footnote{We tested this hypothesis
by re-running the experiment with higher factor rank. For example, selecting factor rank to be 40 gives SNRs of 
16.5, 36.7, 42.4, 75.3 for the corresponding $\eta$ values for the synthetic low-rank experiment in \ref{table5}.} 
\black{ In reality, we do not know the rank value in advance. To make a fair comparison, we used MovieLens (10M) dataset, where we subsampled the available ratings by randomly removing 50\% of the known entries. In this example, we fixed the number of iterations to 100 and compared the SNR and computational time (Table \ref{table6}) for multiple ranks, $k=5,10,20$. It is evident that the we get better SNR in case of LR, also the computational speed of LR is significantly faster then the classic SPG$\ell_1$.}

\begin{table}[ht]
\caption{TFOCS versus classic SPG$\ell_1$ (using direct SVDs) versus LR factorization. {\bf Synthetic low rank} example
shows results for completing a rank 10, $100 \times 100$ matrix, 
with 50\% missing entries. 
SNR, Computational time  and iterations are shown for $\eta=0.1, 0.01, 0.005, 0.0001$.
Rank of the factors is taken to be 10. 
{\bf Seismic} example shows results for matrix completion a low-frequency slice at 10 Hz, extracted from
the Gulf of Suez data set, with 50\% missing entries. SNR, Computational time  and iterations are shown for $\eta=0.2, 0.1, 0.09, 0.08$. Rank of factors was taken to be 28.
}  
\label{table5}
\begin{center}
\black{
\begin{tabular}{cc|c|c|c|c|l}
\cline{3-6}
& & \multicolumn{4}{ c| }{\bf Synthetic low rank} \\ \cline{1-6}
\multicolumn{1}{ |c}{}                        &
\multicolumn{1}{  c|  }{$\eta$} & 0.1 & 0.01 & 0.005 & 0.0001 \\ \cline{1-6}
\multicolumn{1}{ |c| }{\multirow{3}{*}{TFOCS} } &
\multicolumn{1}{ |c| }{SNR (dB)} & 17.2 & 36.3 & 56.2 & 76.2 &     \\ \cline{2-6}
\multicolumn{1}{ |c  }{}                        &
\multicolumn{1}{ |c| }{time (s)} & {24.5} & {179.4} & {963.3} & {2499.9} &     \\ \cline{2-6}
\multicolumn{1}{ |c  }{}                        &
\multicolumn{1}{ |c| }{ iteration} & {1151 } & {8751} & { 46701} & {121901} &     \\ \cline{1-6}
\multicolumn{1}{ |c  }{\multirow{3}{*}{SPG$\ell_1$} } &
\multicolumn{1}{ |c| }{SNR (dB)} & 14.5 & 36.4 & 39.2&76.2 &\\ \cline{2-6}
\multicolumn{1}{ |c  }{}                        &
\multicolumn{1}{ |c| }{ time (s)} & {4.9} & { 17.0} & { 17.2} & {61.1} &  \\ \cline{2-6}
\multicolumn{1}{ |c  }{}                        &
\multicolumn{1}{ |c| }{ iteration} & {12} & {46 } & {47} & { 152} &     \\ \cline{1-6}
\multicolumn{1}{ |c  }{\multirow{3}{*}{LR} } &
\multicolumn{1}{ |c| }{SNR (dB)} & 16.5 & 36.7 & 42.7 & 76.2 \\ \cline{2-6}
\multicolumn{1}{ |c  }{}                        &
\multicolumn{1}{ |c| }{time (s)} &{ 0.6} & {0.5} & { 0.58} &{0.9} \\ \cline{2-6}
\multicolumn{1}{ |c  }{}                        &
\multicolumn{1}{ |c| }{ iteration} & {27} & { 64} & {73} & { 119} &     \\ \cline{1-6}
\end{tabular}\begin{tabular}{cc|c|c|c|c|l}
\cline{3-6}
& & \multicolumn{4}{ c| }{\bf Seismic} \\ \cline{1-6}
\multicolumn{1}{ |c}{}                        &
\multicolumn{1}{  c|  }{$\eta$}  & 0.2 & 0.1 & 0.09 & 0.08 \\ \cline{1-6}
\multicolumn{1}{ |c| }{\multirow{3}{*}{TFOCS} } &
\multicolumn{1}{ |c| }{SNR (dB)} & 13.05 & 17.4 & 17.9 & 18.5 &     \\ \cline{2-6}
\multicolumn{1}{ |c  }{}                        &
\multicolumn{1}{ |c| }{ time (s)} &{ 593.3 }& {3232.3 }& { 4295.1} &{6140.2} &     \\ \cline{2-6}
\multicolumn{1}{ |c  }{}                        &
\multicolumn{1}{ |c| }{iteration} & {1201 } & {3395} & {3901} & {4451} &     \\ \cline{1-6}
\multicolumn{1}{ |c  }{\multirow{3}{*}{SPG$\ell_1$} } &
\multicolumn{1}{ |c| }{SNR (dB)} &12.8  & 17.0 &17.4 &17.9& \\ \cline{2-6}
\multicolumn{1}{ |c  }{}                        &
\multicolumn{1}{ |c| }{ time (s)} &{ 30.4 }&{ 42.8}  &{32.9}  &{ 58.8}&  \\ \cline{2-6}
\multicolumn{1}{ |c  }{}                        &
\multicolumn{1}{ |c| }{ iteration} & {37 } & { 52} & {40} & {73} &     \\ \cline{1-6}
\multicolumn{1}{ |c  }{\multirow{3}{*}{LR} } &
\multicolumn{1}{ |c| }{SNR (dB)} &13.1 & 17.1 & 17.4 & 18.0 \\ \cline{2-6}
\multicolumn{1}{ |c  }{}                        &
\multicolumn{1}{ |c| }{ time (s)} & {1.6} & {2.9} & {3.2} & {4.0} \\ \cline{2-6}
\multicolumn{1}{ |c  }{}                        &
\multicolumn{1}{ |c| }{ iteration} & {38} & {80} & {87} & {113 } &     \\ \cline{1-6}
\end{tabular}
}
\end{center}
\end{table}

\subsubsection{Simultaneous missing-trace interpolation and denoising}
\label{sec:RobustCompletion}
To illustrate the utility of robust cost functions, we consider a situation where observed data are heavily contaminated. 
The goal here is to simultaneously denoise interpolate the data. 
We work with same seismic line from Gulf of Suez. To obtain the observed data, we apply a sub-sampling mask that randomly removes 50\% of the shots, and to simulate contamination, we replace another 10\% of the shots with large random errors, whose amplitudes are three times the maximum amplitude present in the data. 
\black{In reality, we know the sub-sampling mask but we do not know the behaviour and amplitude of noise.}
In this example, we formulate and solve the robust matrix completion problem~\eqref{BPDN}, where 
the cost $\rho$ is taken to be the penalty~\eqref{StudentForm}; see section~\ref{sec:Robust} for the explanation 
and motivation.  
As in the previous examples, the recovery is done in the
(m-h) domain. We implement the formulation in the frequency domain, where we work with monochromatic frequency slices, and adjust the rank and $\nu$ parameter while going from low to high frequency slices. Figure \ref{fig:9} compares the recovery results with and without using a robust penalty function. The error budget plays a significant role in this example, and we standardized the problems 
by setting the {\it relative error} to be 20\% of the initial error, so that the formulations are comparable.  

We can clearly see that the standard least squares formulation is unable to recover the true solution. 
The intuitive reason is that the least squares penalty is simply unable to budget large errors to what should be the outlying residuals. 
The Student's t penalty, in contrast, achieves a good recovery in this extreme situation, with an SNR of 17.9 DB. 
In this example, we used 300 iterations of SPG$\ell_1$ for all frequency slices.

\subsubsection{Re-Weighting}
\label{sec:Reweighting}
Re-weighting for seismic trace interpolation was recently used in \cite{mansour12iwr} to improve the interpolation of subsampled seismic traces in the context 
of sparsity promotion in the curvelet domain. The weighted $\ell_1$ formulation takes advantage of curvelet support overlap across adjacent frequency slices.

Analogously, in the matrix setting, we use the weighted rank-minimization formulation~\eqref{wBPDN} to take advantage of correlated row and column subspaces for adjacent frequency slices. We first demonstrate the effectiveness of solving the \eqref{wBPDN} problem when we have accurate subspace information. For this purpose, we compute the row and column subspace bases of the fully sampled low frequency (11Hz) seismic slice and pass this information to ~\eqref{wBPDN}  using matrices $Q$ and $W$. Figures~\ref{fig:10}(a) and (b) show the residual of the frequency slice with and without weighting. The reconstruction using the~\eqref{wBPDN} problem achieves a 1.5dB improvement in SNR over the non-weighted~\eqref{BPDN} formulation. 

Next, we apply the \eqref{wBPDN} formulation in a practical setting where we do not know 
subspace bases ahead of time, but learn them as we proceed from low to high 
frequencies. We use the row and column subspace vectors recovered 
using~\eqref{BPDN} for 10.75 Hz and 15.75 Hz frequency slices as subspace estimates for the adjacent higher frequency slices at 11 Hz and 16 Hz. 
Using the~\eqref{wBPDN} formulation in this way yields SNR improvements
of 0.6dB and 1dB, respectively, over~\eqref{BPDN} alone.
Figures~\ref{fig:11}(a) and (b) show the residual for the next higher frequency without using the support and Figures~\ref{fig:11}(c) and (d) shows the residual for next higher frequency with support from previous frequency. Figure~\ref{fig:12} shows the 
recovery SNR versus frequency for weighted and non-weighted cases for 
a range of frequencies from 9 Hz to 17 Hz.

\begin{figure}
  \begin{center}
    \subfigure[]{\includegraphics[scale=0.35]{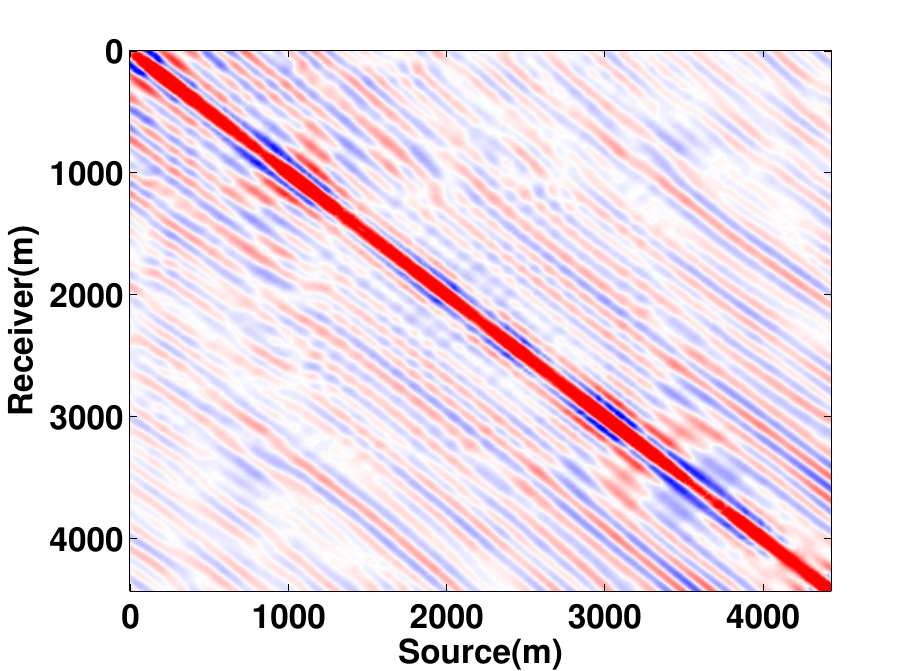}}
    \subfigure[]{\includegraphics[scale=0.35]{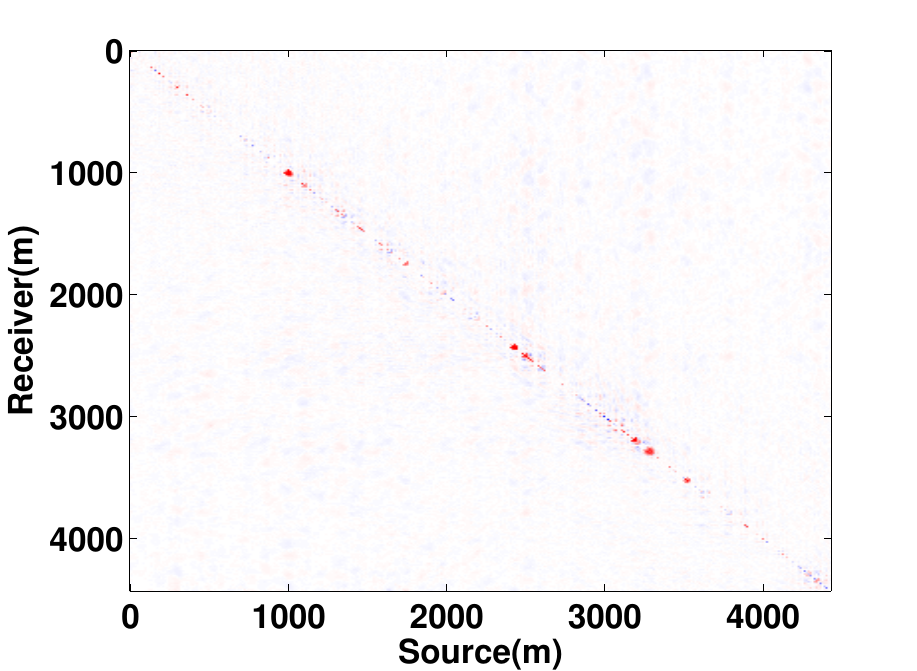}}
     \subfigure[]{\includegraphics[scale=0.35]{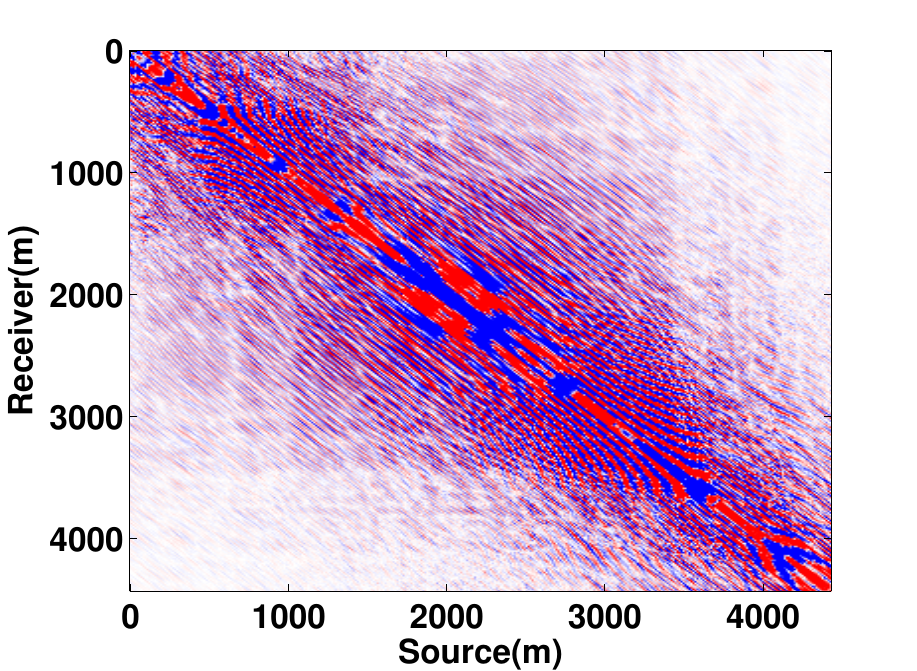}}
    \subfigure[]{\includegraphics[scale=0.35]{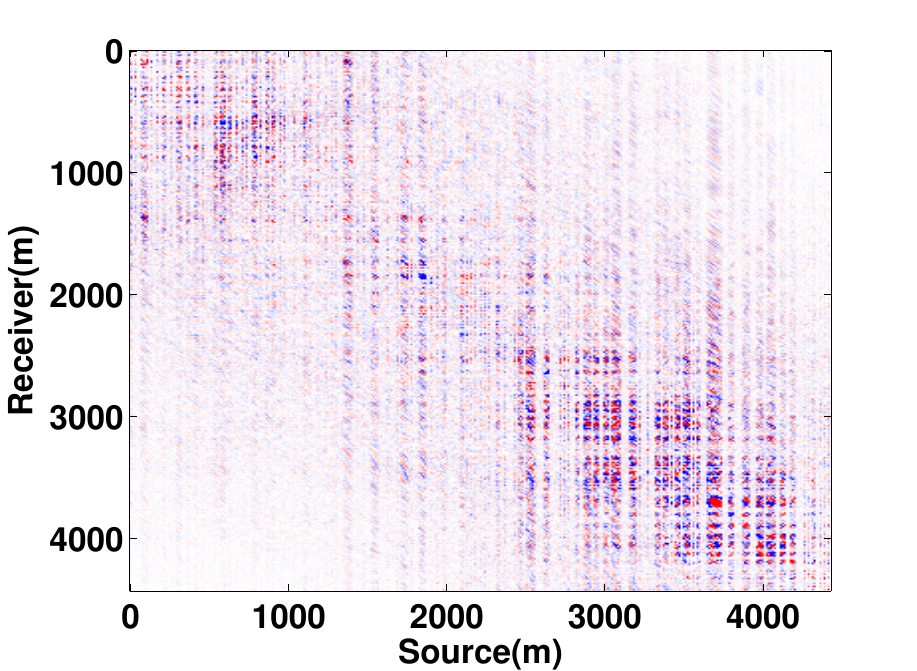}}
  \end{center}
  \caption{Recovery results for 50\% subsampled 2D frequency slices using the nuclear norm formulation. 
  (a) Interpolation and (b) residual of low frequency slice at 12 Hz with SNR = 19.1 dB. (c) Interpolation and (d) residual of high frequency slice at 60 Hz with SNR = 15.2 dB.}
  \label{fig:3}
\end{figure}

\begin{figure}
  \begin{center}
    \subfigure[]{\includegraphics[scale=0.35]{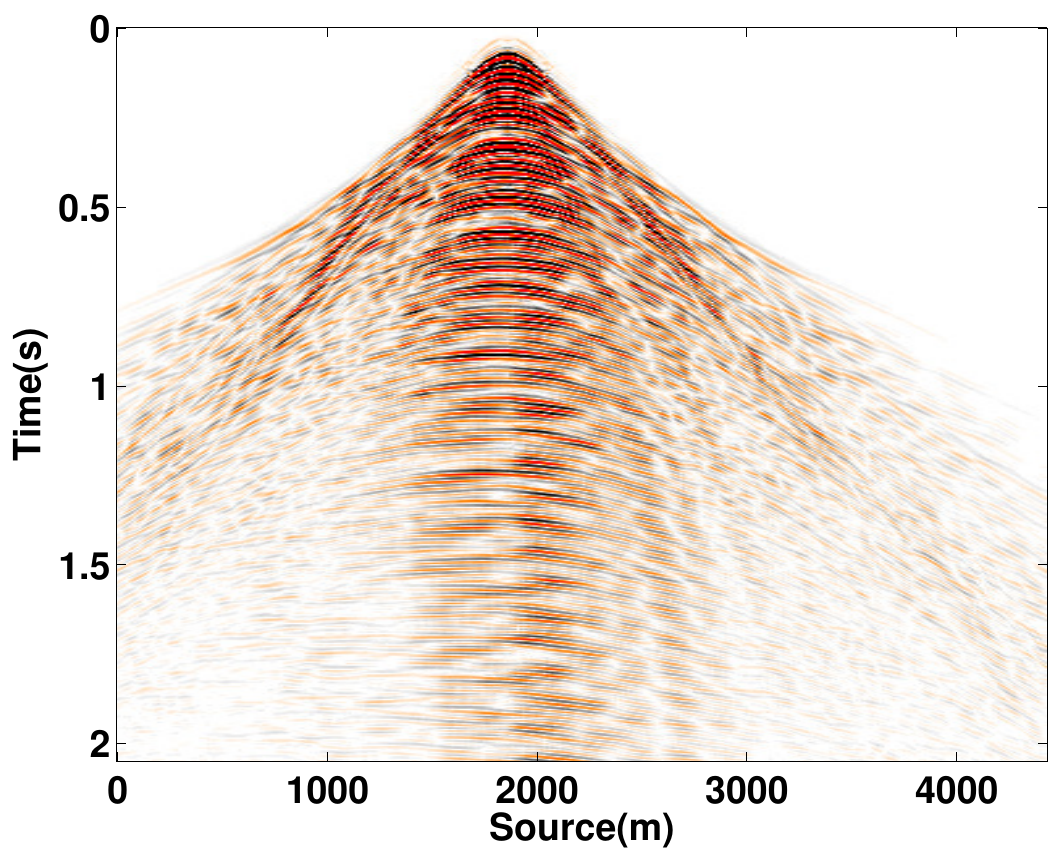}}
     \subfigure[]{\includegraphics[scale=0.35]{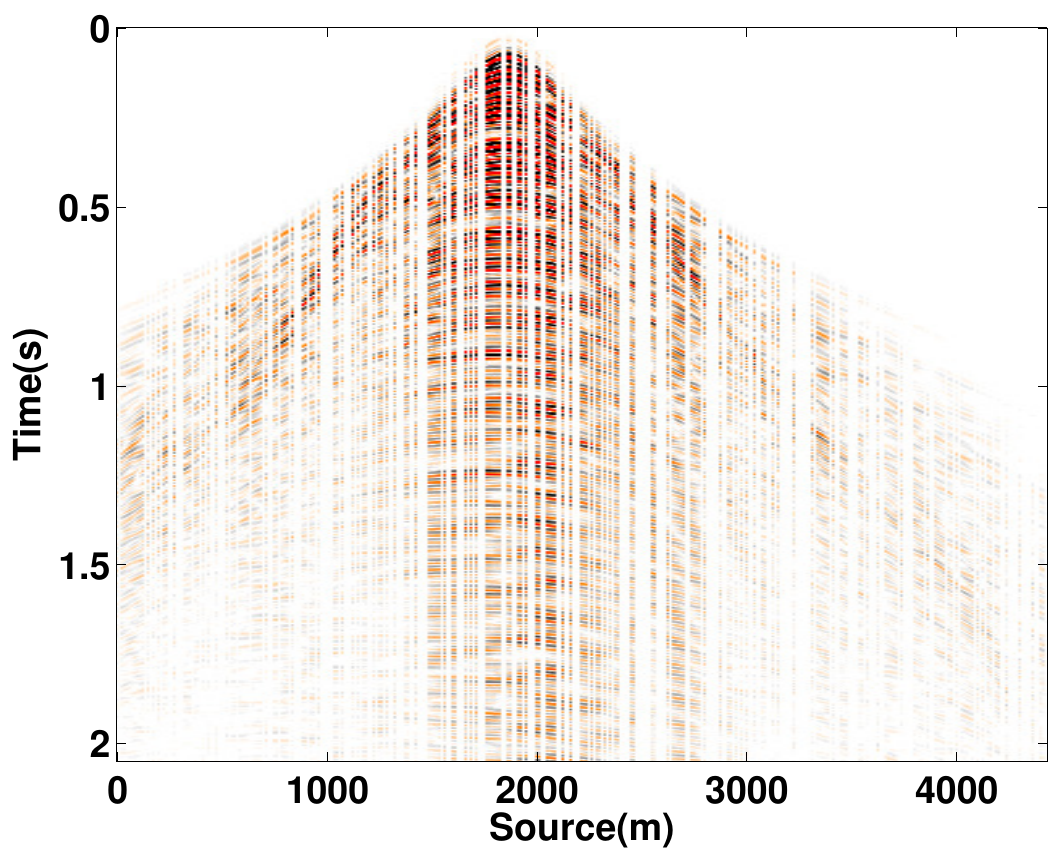}}
    \subfigure[]{\includegraphics[scale=0.35]{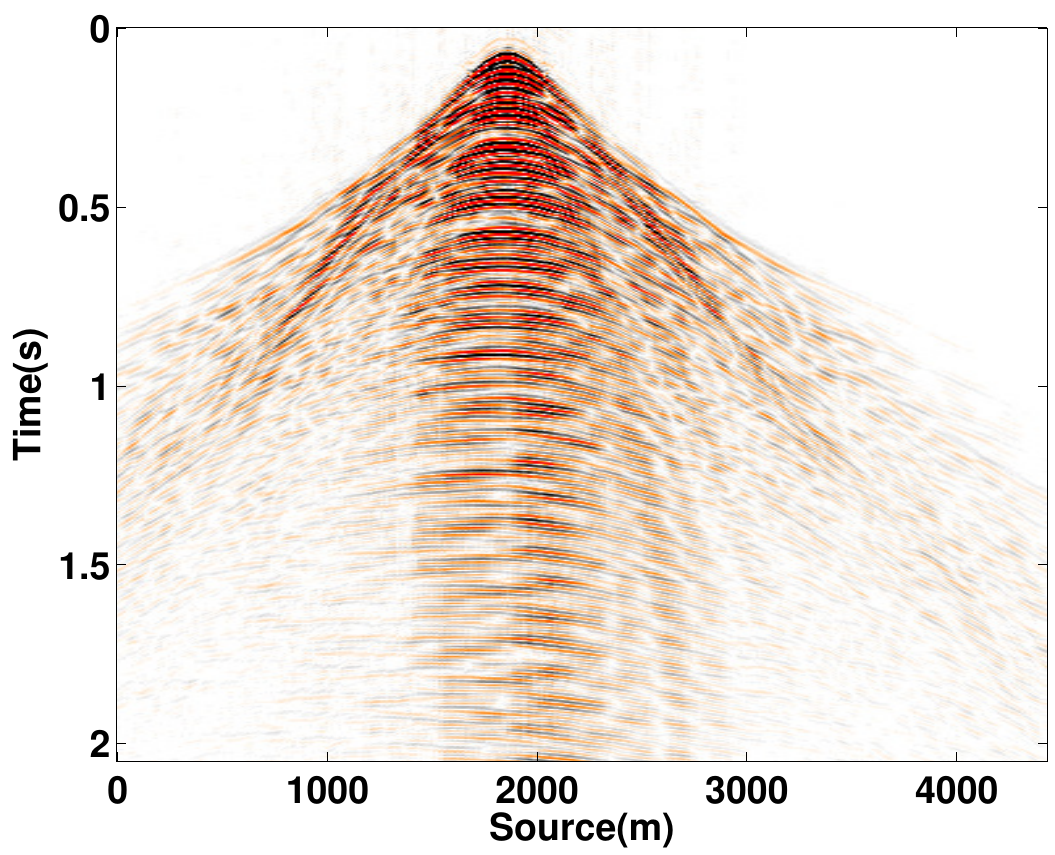}}
    \subfigure[]{\includegraphics[scale=0.35]{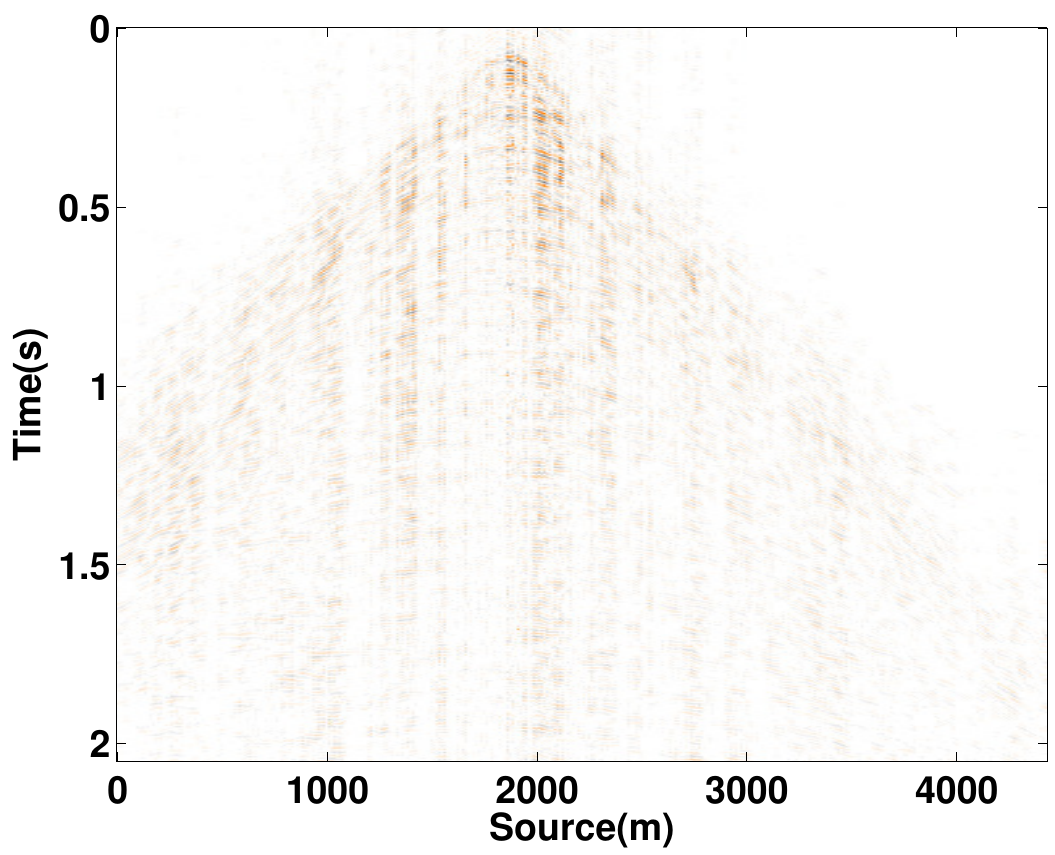}}
  \end{center}
  \caption{Missing trace interpolation of a seismic line from Gulf of Suez.
  (a) Ground truth. (b) $50\%$ subsampled common shot gather.
  (c) Recovery result with a SNR of 18.5 dB.
  (d) Residual.}  
  \label{fig:4}
\end{figure}

\begin{figure}
  \begin{center}
      \subfigure[]{\includegraphics[scale=0.32]{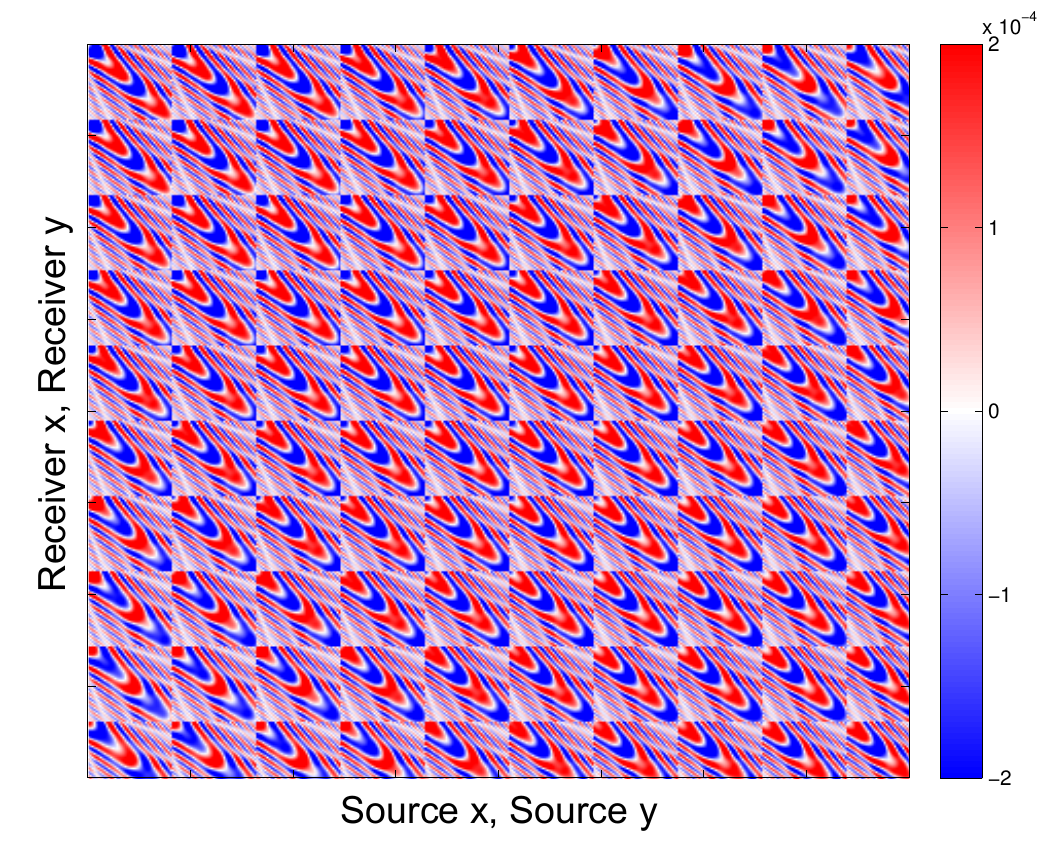}}
      \subfigure[]{\includegraphics[scale=0.32]{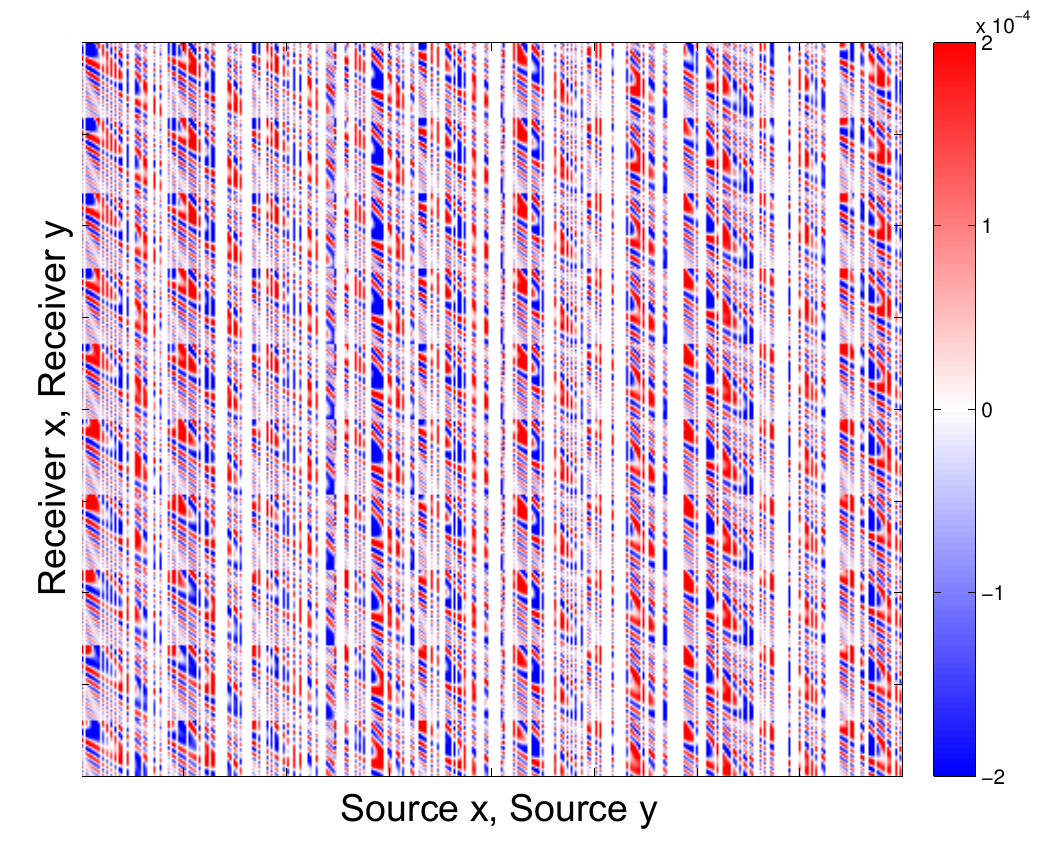}}
       \subfigure[]{\includegraphics[scale=0.32]{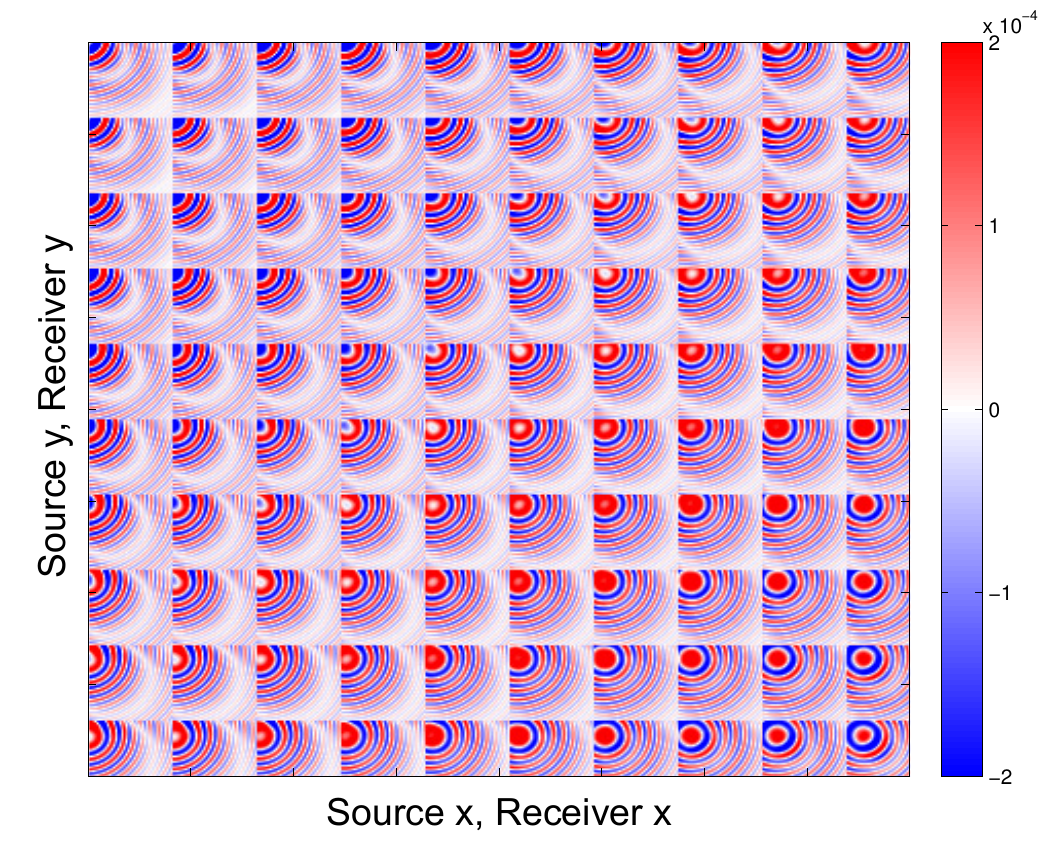}}
      \subfigure[]{\includegraphics[scale=0.32]{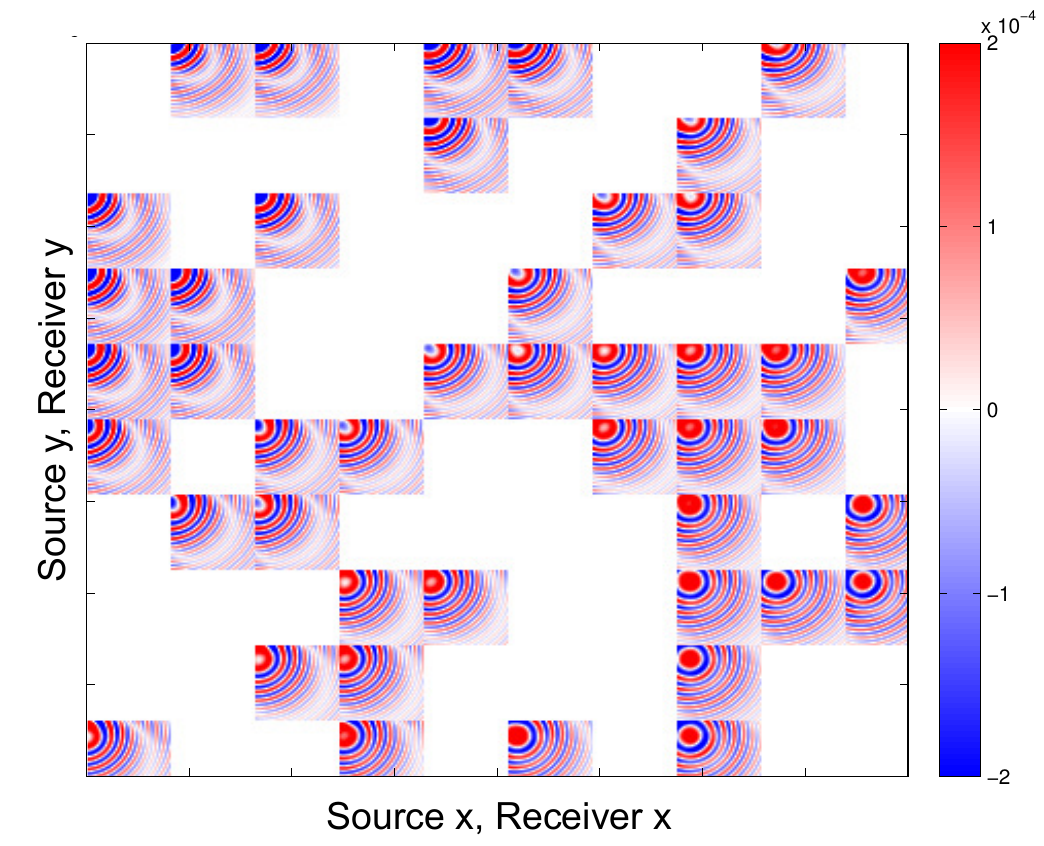}}
  \end{center}
  \caption{Matricization of 4D monochromatic frequency slice. Top: (Source x, Source y) matricization. Bottom: (Source x, Receiver x) matricization.
                  Left: Fully sampled data; Right: Subsampled data.}
  \label{fig:5}
\end{figure}

\begin{figure}
  \begin{center}
        \subfigure[]{\includegraphics[scale=0.32]{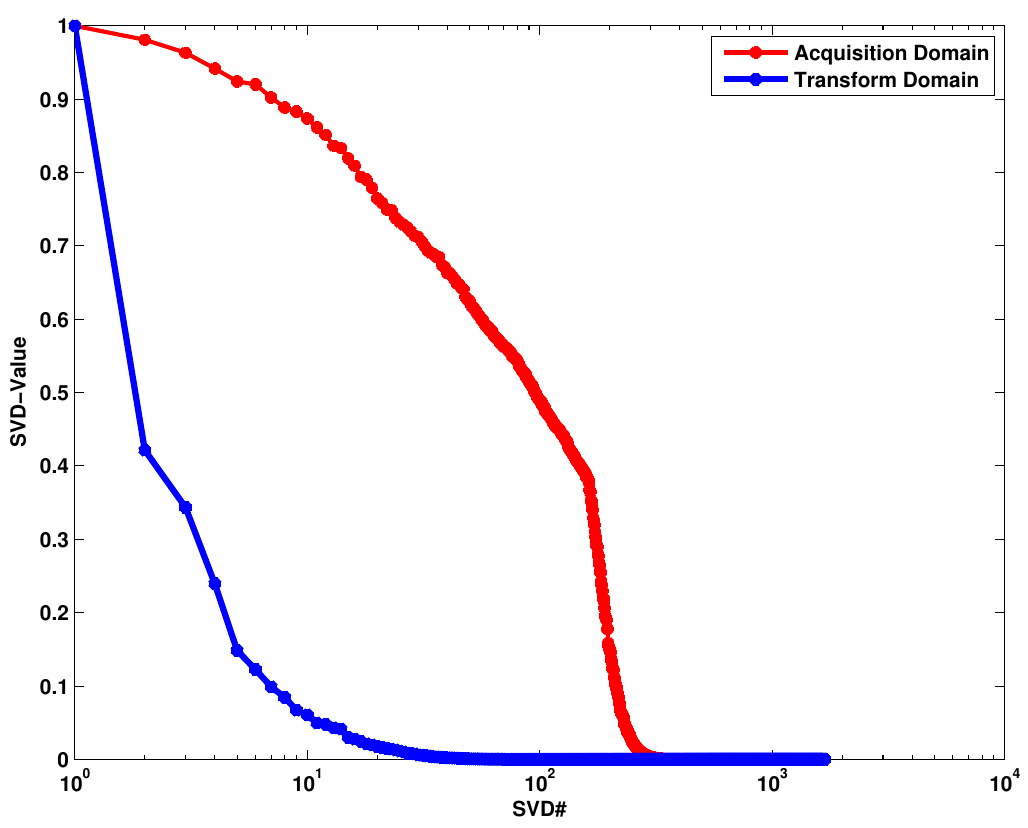}}
      \subfigure[]{\includegraphics[scale=0.32]{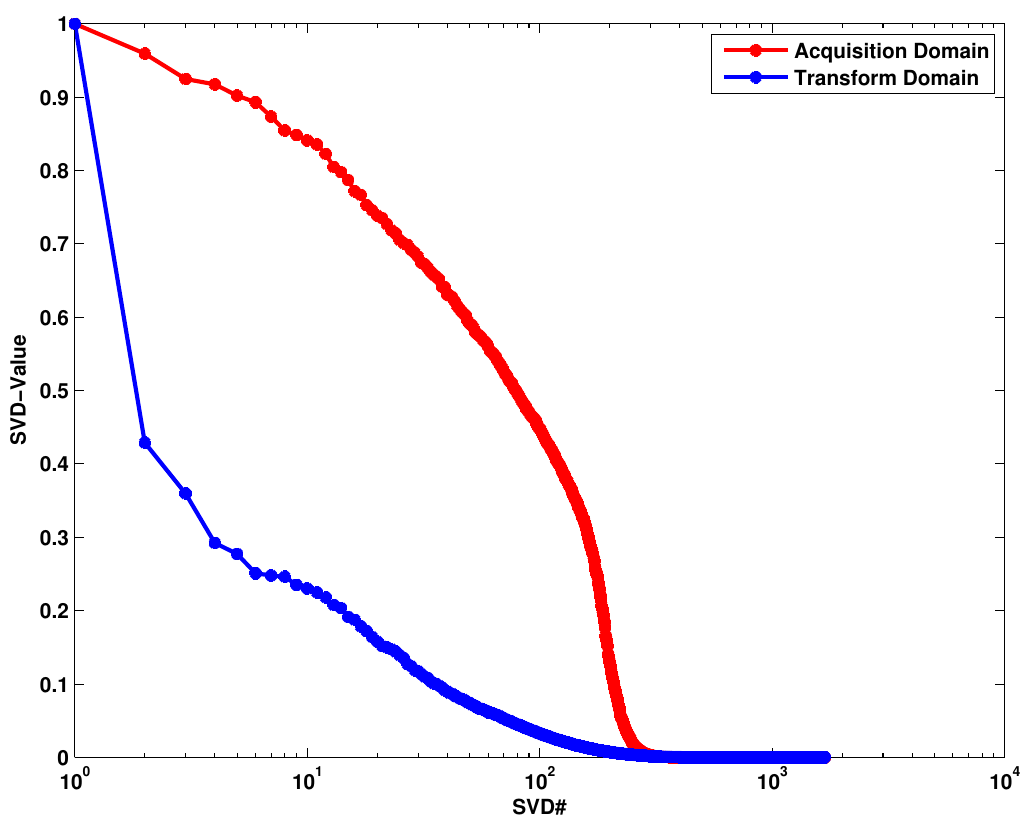}}
  \end{center}
  \caption{Singular value decay in case of different matricization of 4D monochromatic frequency slice.
                  Left: Fully sampled data; Right: Subsampled data.}
  \label{fig:13}
\end{figure}

\begin{figure}
  \begin{center}
      \subfigure[]{\includegraphics[scale=0.3]{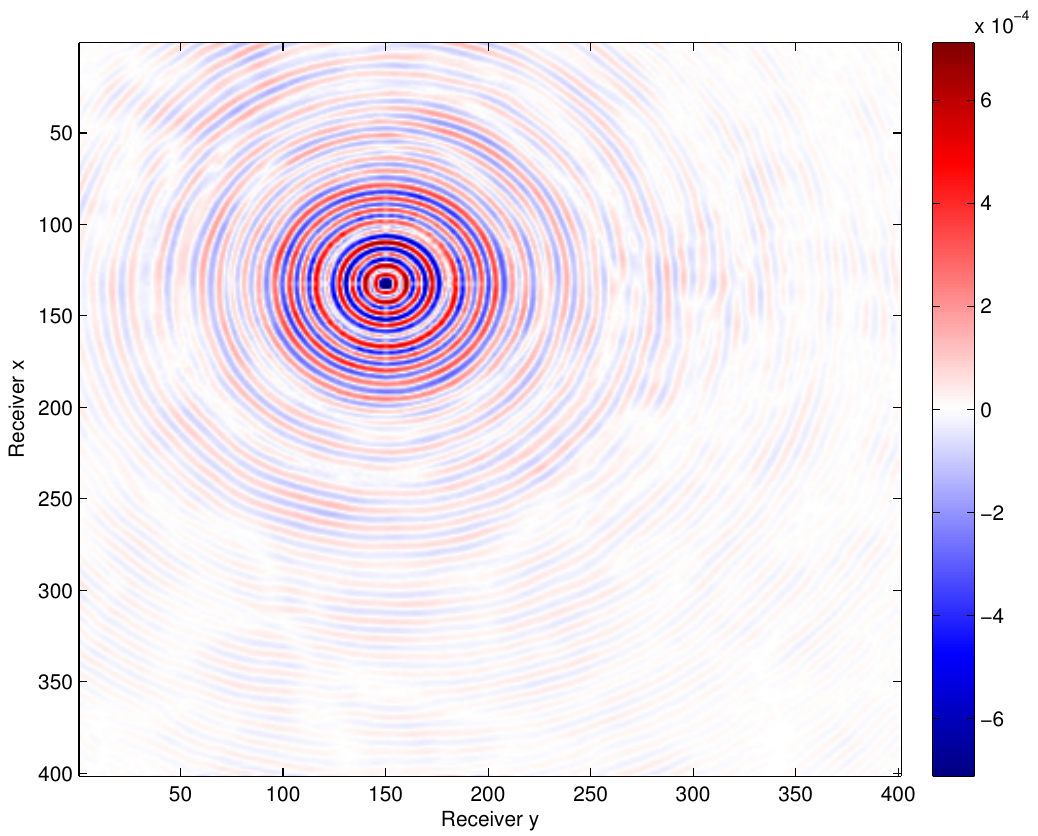}}
      \subfigure[]{\includegraphics[scale=0.3]{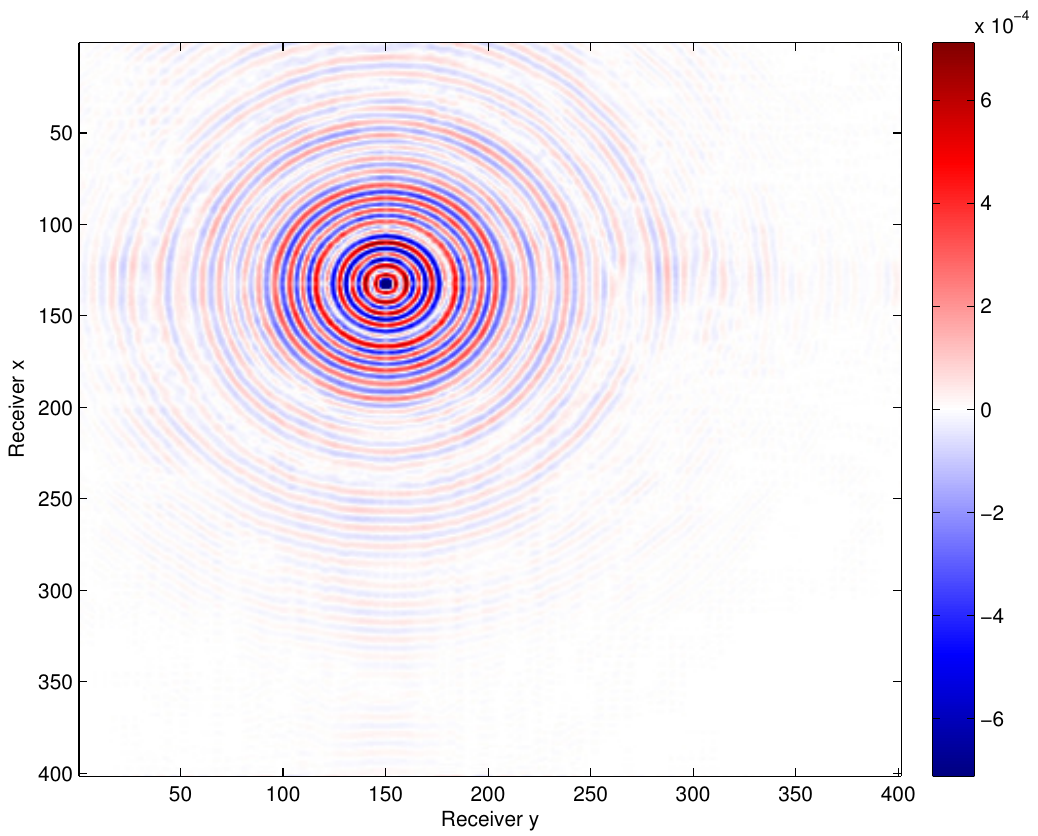}}
      \subfigure[]{\includegraphics[scale=0.3]{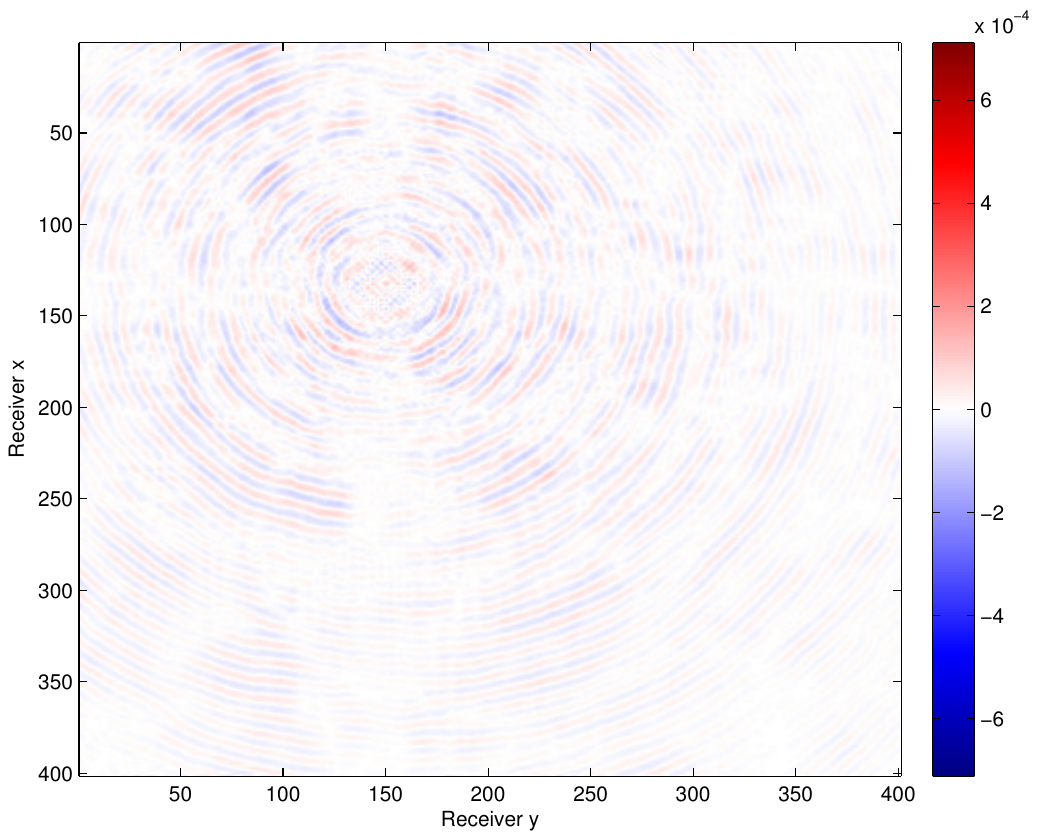}}
      \subfigure[]{\includegraphics[scale=0.3]{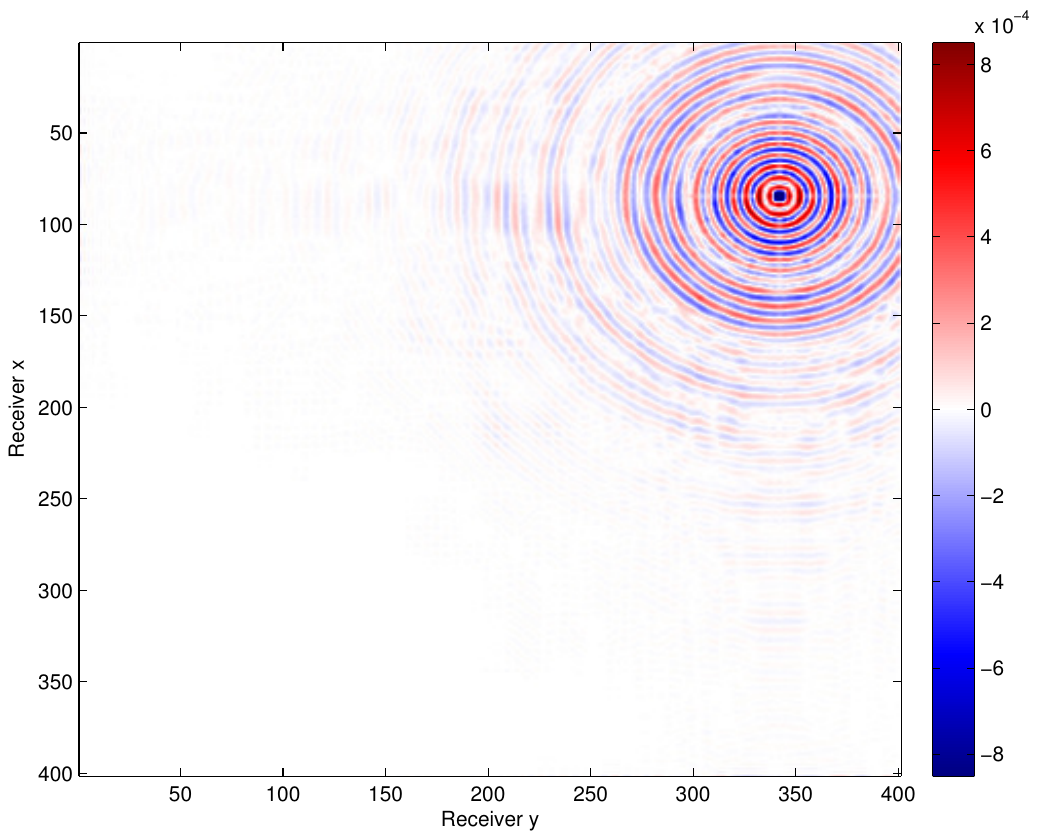}}
      \subfigure[]{\includegraphics[scale=0.3]{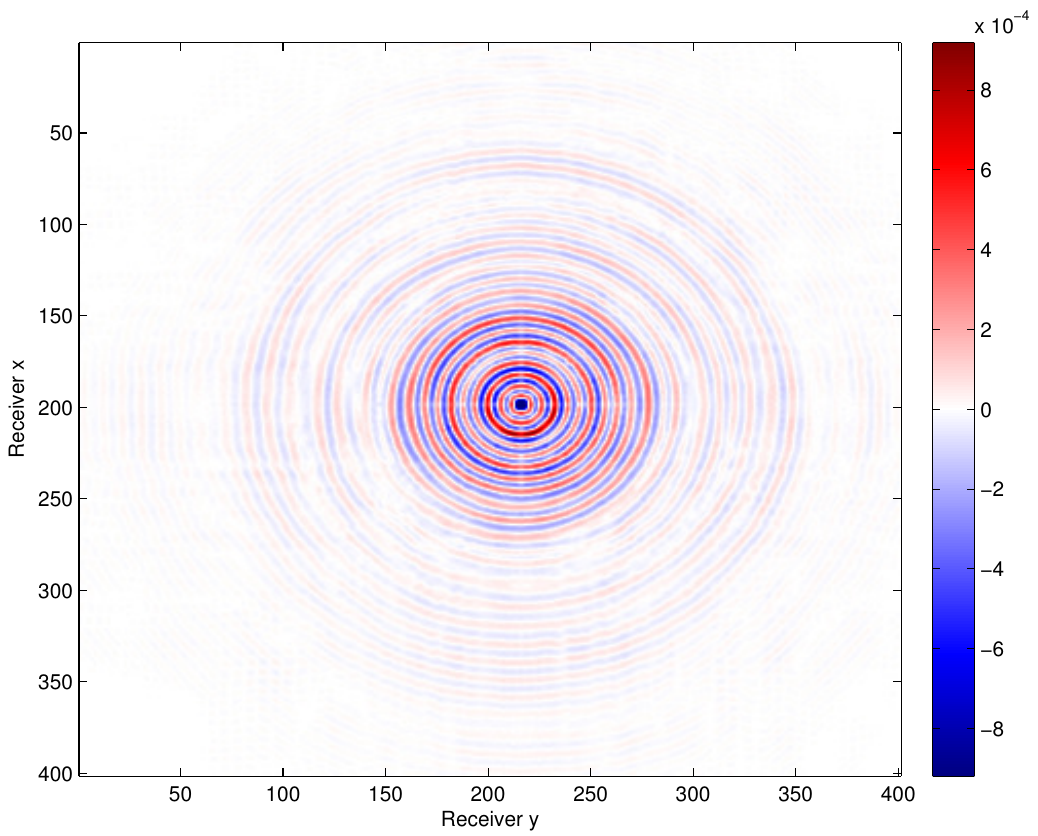}}
      \subfigure[]{\includegraphics[scale=0.3]{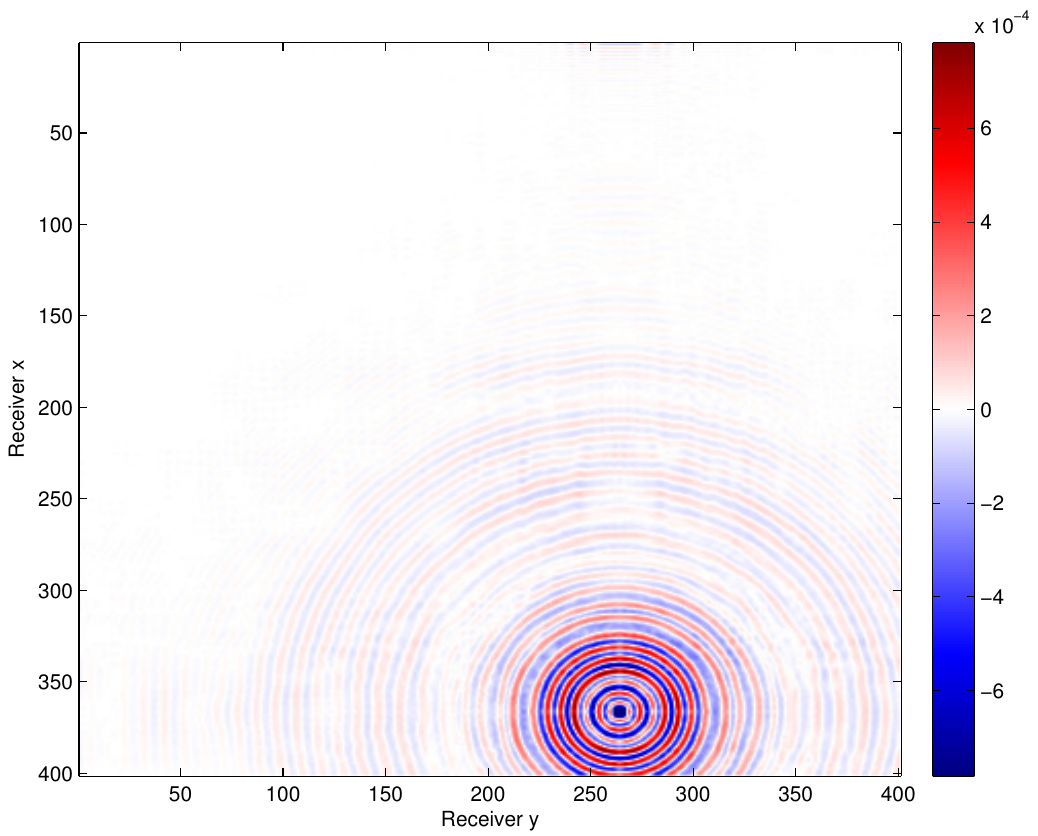}}
  \end{center}
  \caption{Missing-trace interpolation of a frequency slice at 12.3Hz extracted from 5D data set, \black{$75\%$ missing data}. (a,b,c) Original, recovery and residual of a common shot gather with a SNR of 11.4 dB at the location
where shot is recorded. (d,e,f) Interpolation of common shot gathers at the location where no reference
shot is present.}
  \label{fig:6}
\end{figure}

\begin{figure}
  \begin{center}
      \subfigure[]{\includegraphics[scale=0.3]{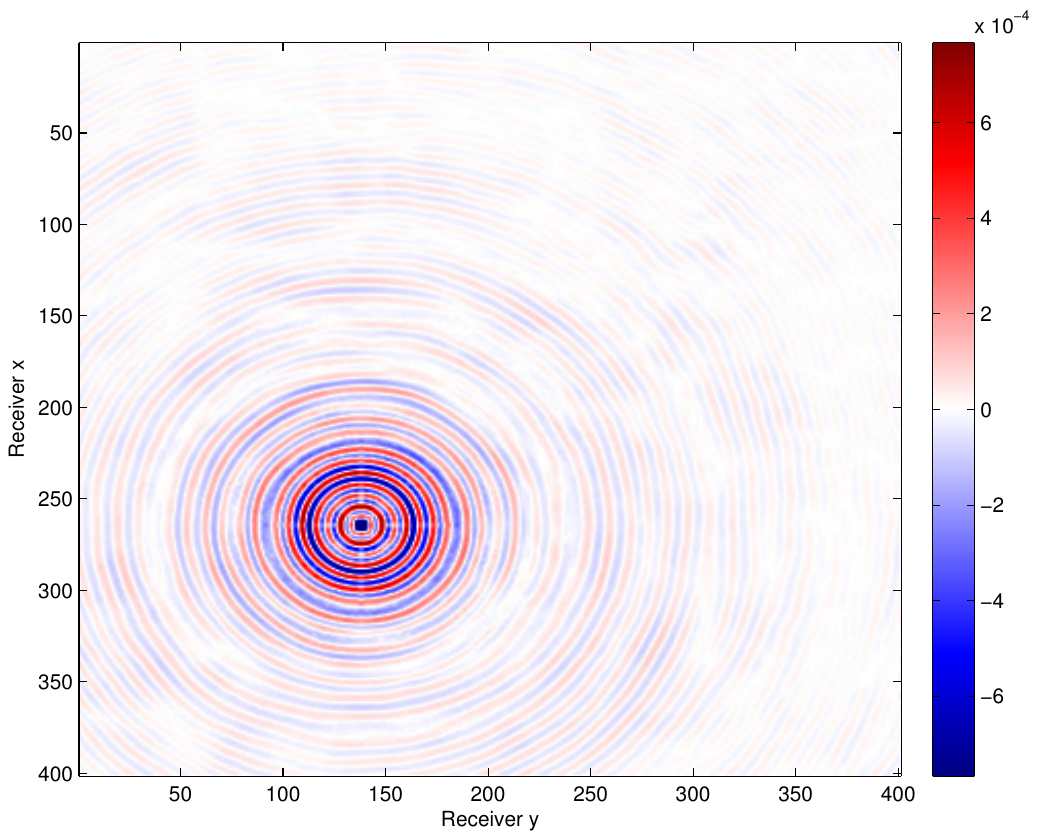}}
      \subfigure[]{\includegraphics[scale=0.3]{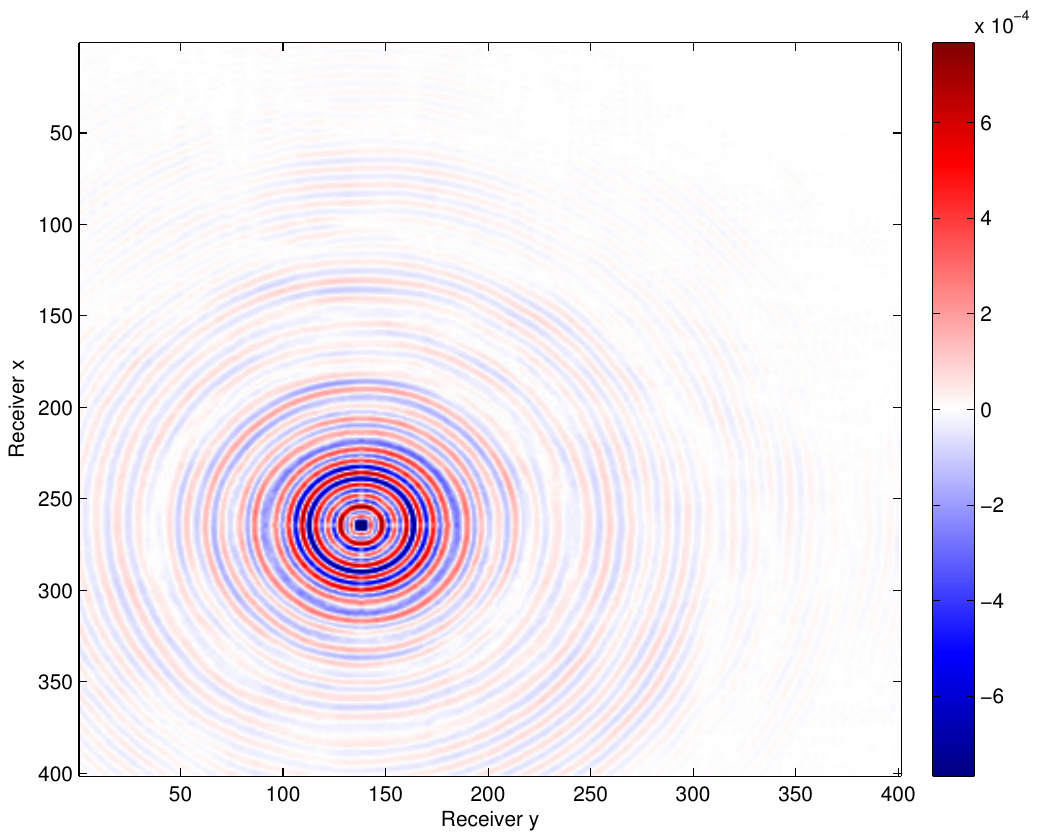}}
      \subfigure[]{\includegraphics[scale=0.3]{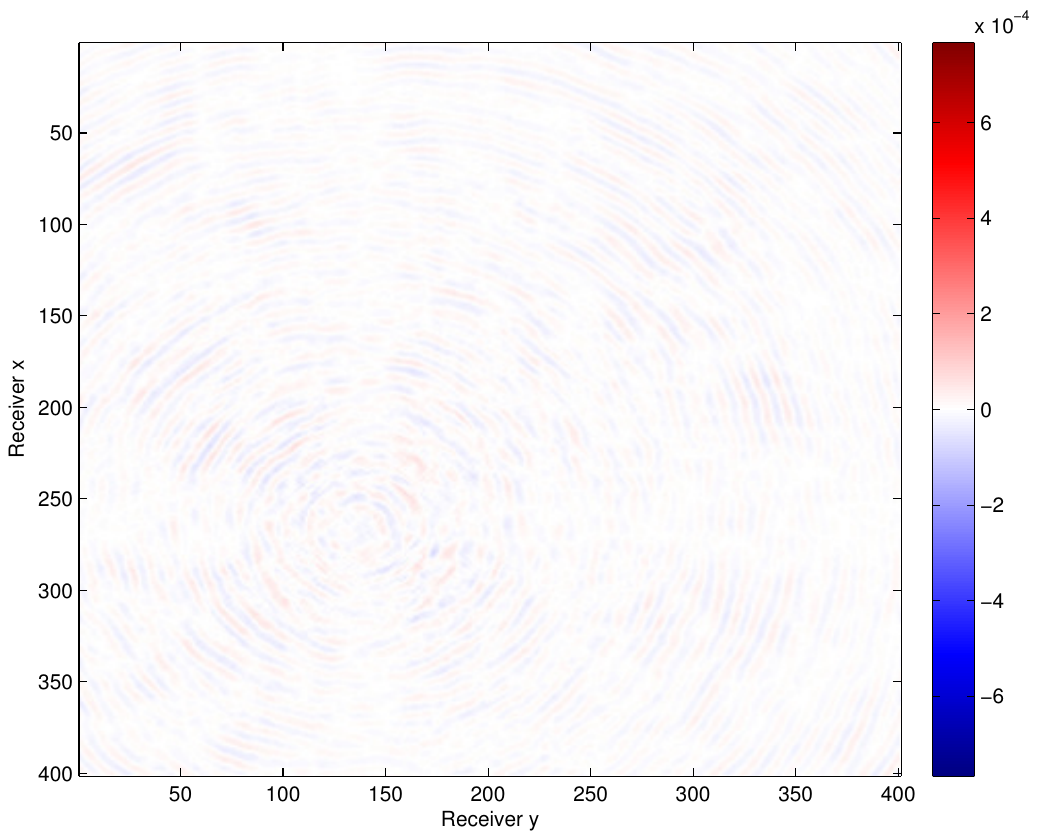}}
      \subfigure[]{\includegraphics[scale=0.3]{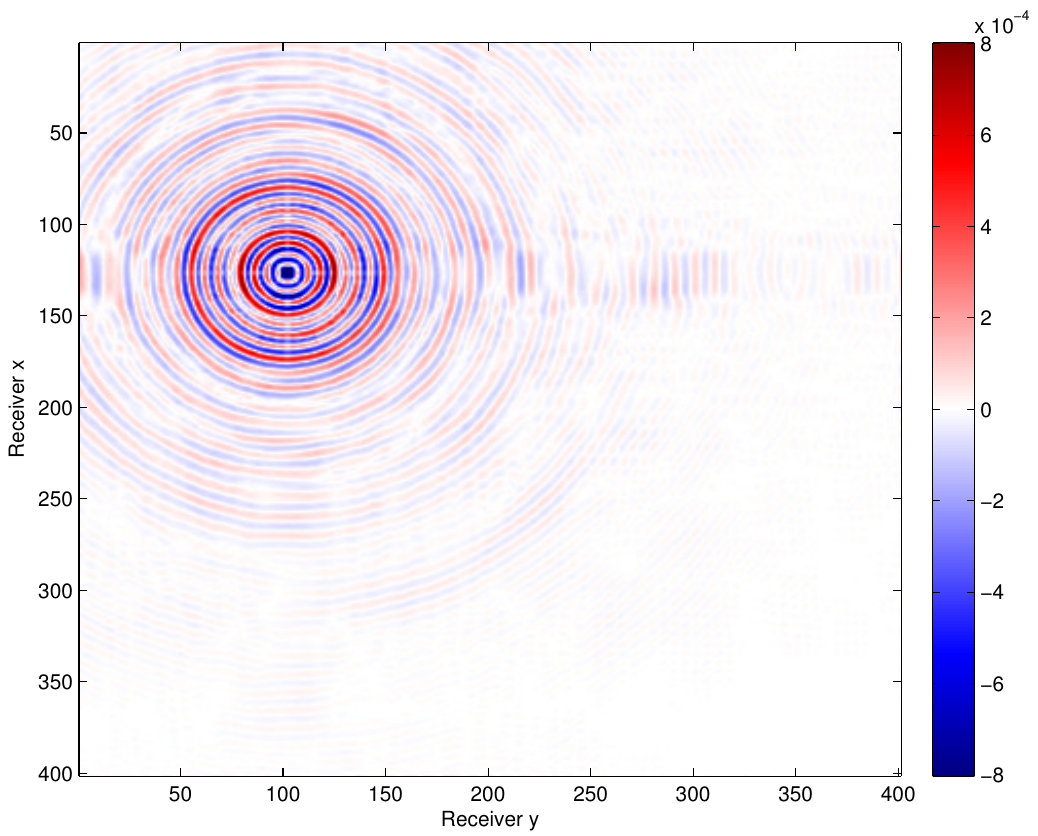}}
      \subfigure[]{\includegraphics[scale=0.3]{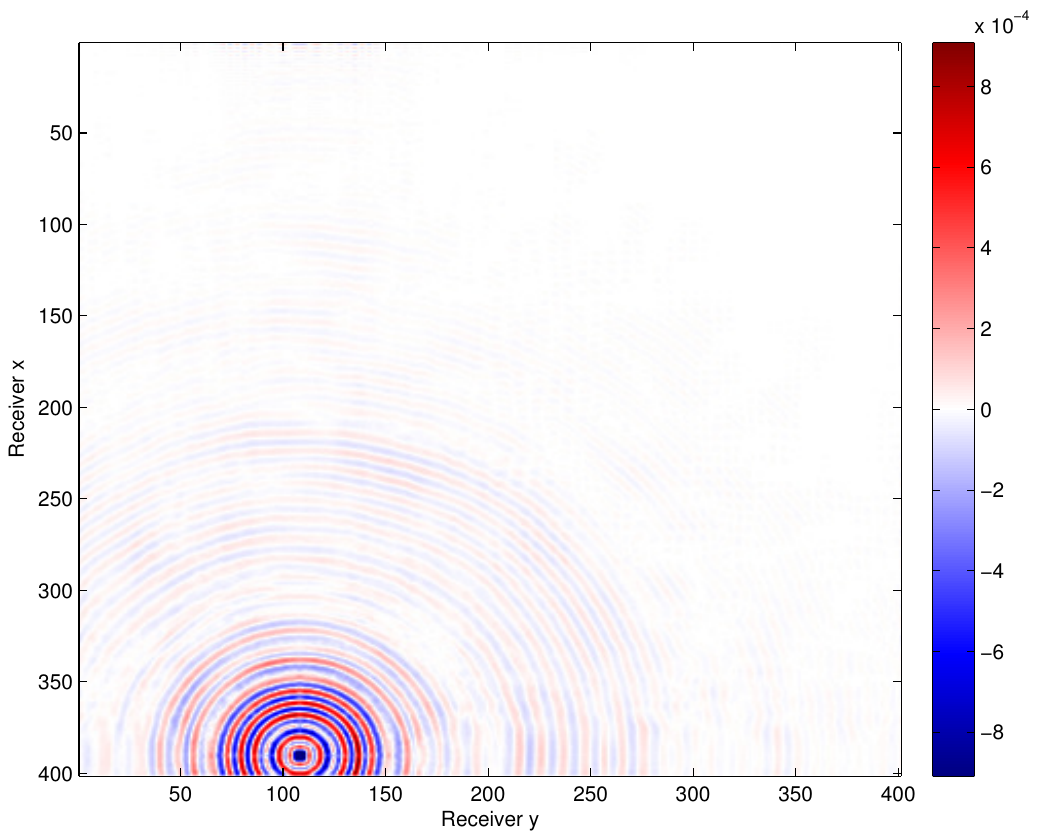}}
      \subfigure[]{\includegraphics[scale=0.3]{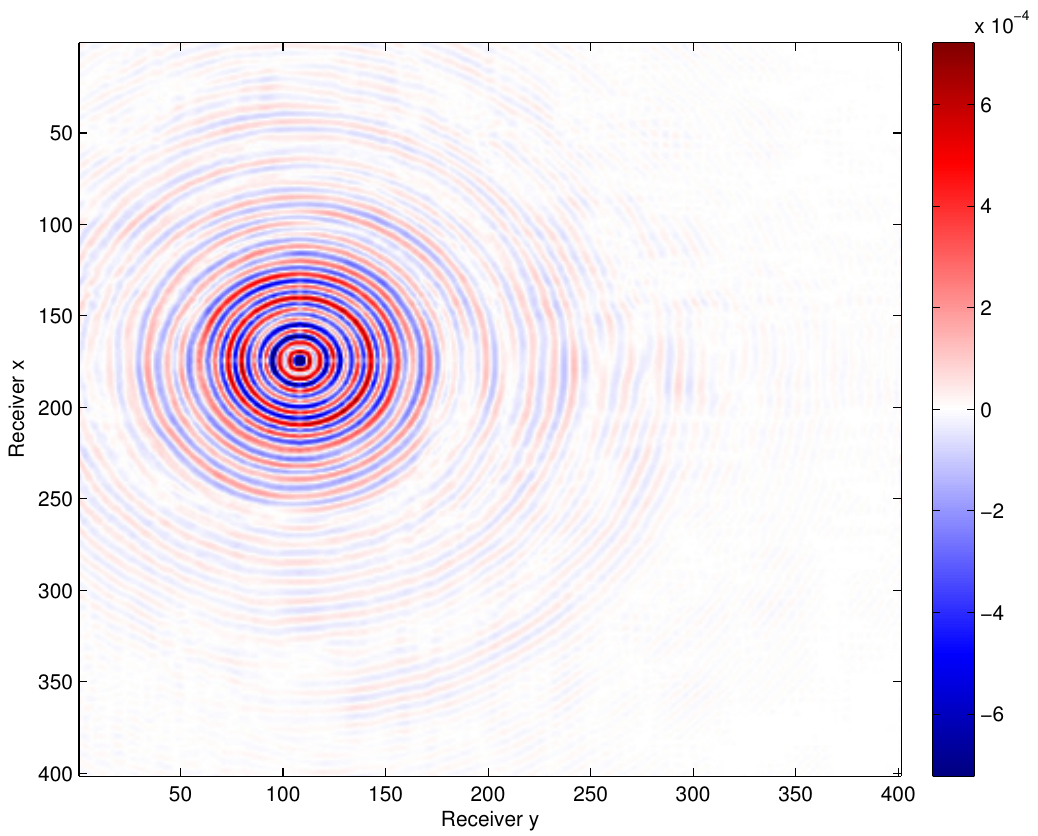}}
  \end{center}
  \caption{Missing-trace interpolation of a frequency slice at 12.3Hz extracted from 5D data set, \black{$50\%$ missing data}. (a,b,c) Original, recovery and residual of a common shot gather with a SNR of 16.6 dB at the location
where shot is recorded. (d,e,f) Interpolation of common shot gathers at the location where no reference
shot is present.}
  \label{fig:7}
\end{figure}

\begin{figure}
  \begin{center}
    \subfigure[]{\includegraphics[scale=0.4]{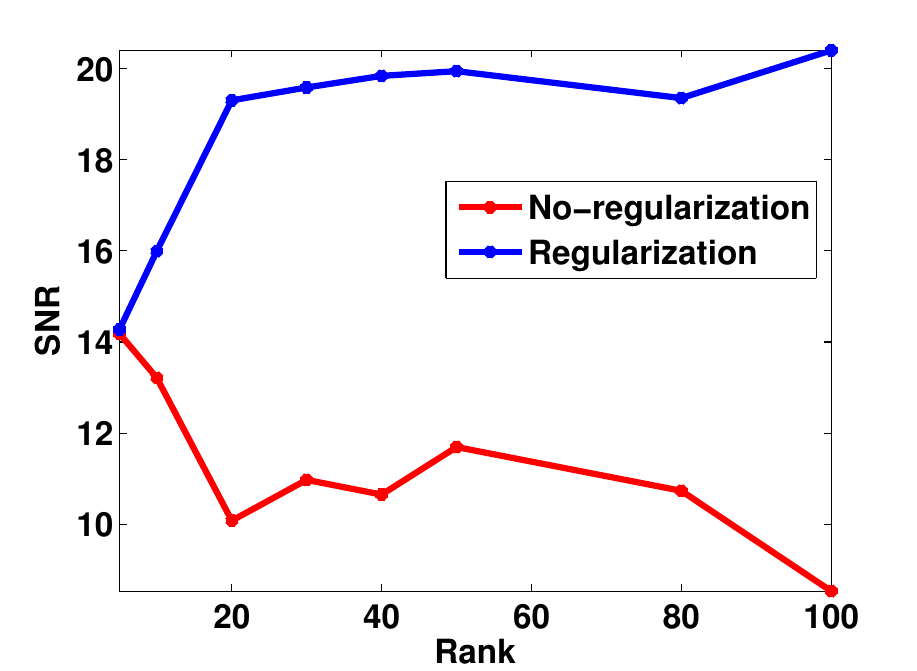}}
    \subfigure[]{\includegraphics[scale=0.4]{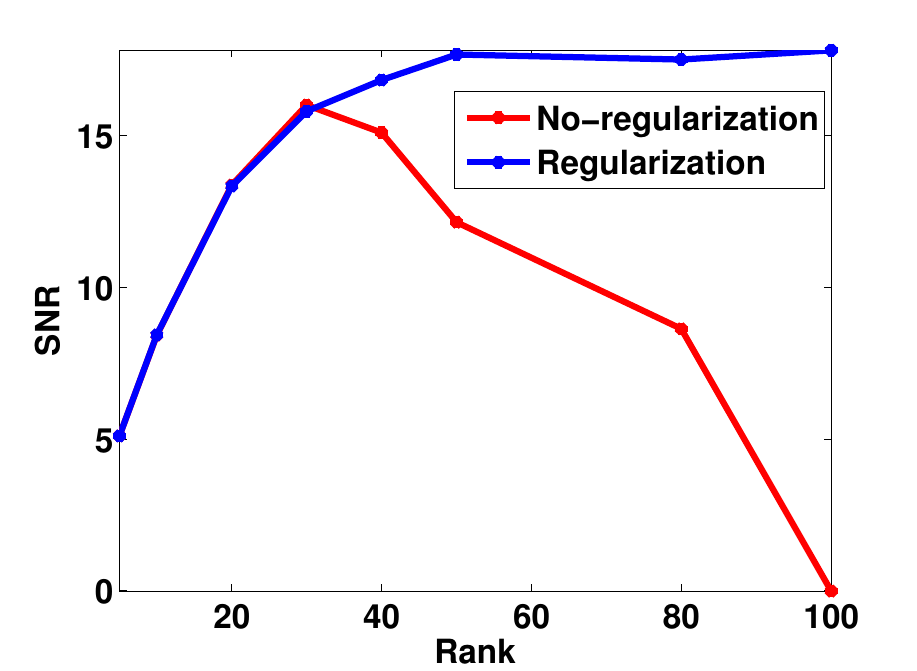}}
  \end{center}
  \caption{Comparison of regularized and non-regularized formulations. SNR of 
  (a) low frequency slice at 12 Hz and 
  (b) high frequency slice at 60 Hz over a range of factor ranks. 
  Without regularization, recovery quality decays with factor rank 
  due to  over-fiting; the regularized formulation improves with higher factor rank. }
  \label{fig:8}
\end{figure}

\begin{figure}
  \begin{center}
    \subfigure[]{\includegraphics[scale=0.35]{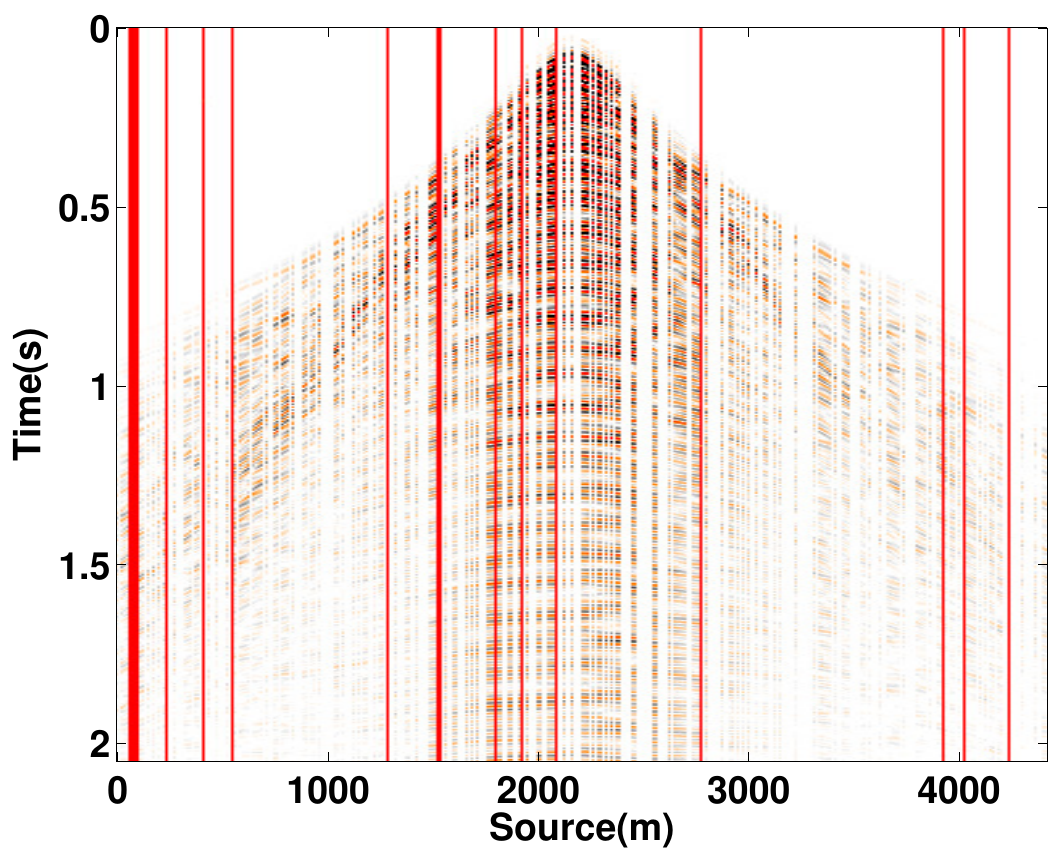}}
     \subfigure[]{\includegraphics[scale=0.35]{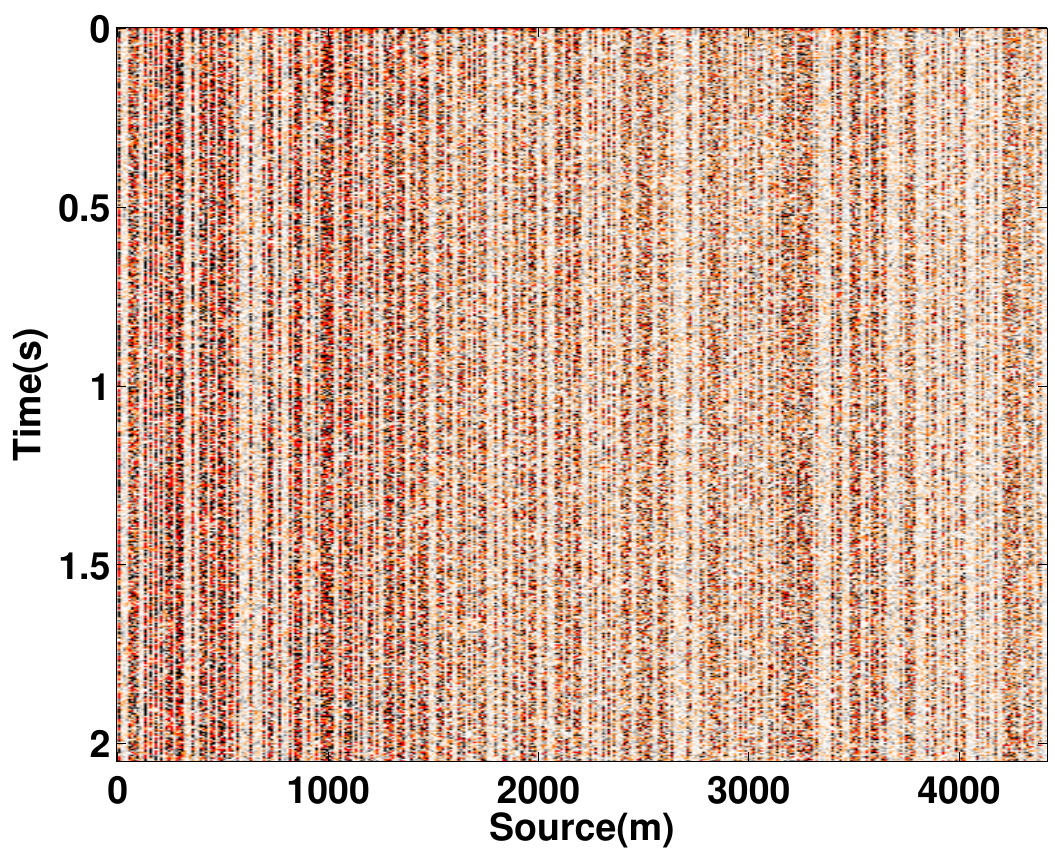}}
    \subfigure[]{\includegraphics[scale=0.35]{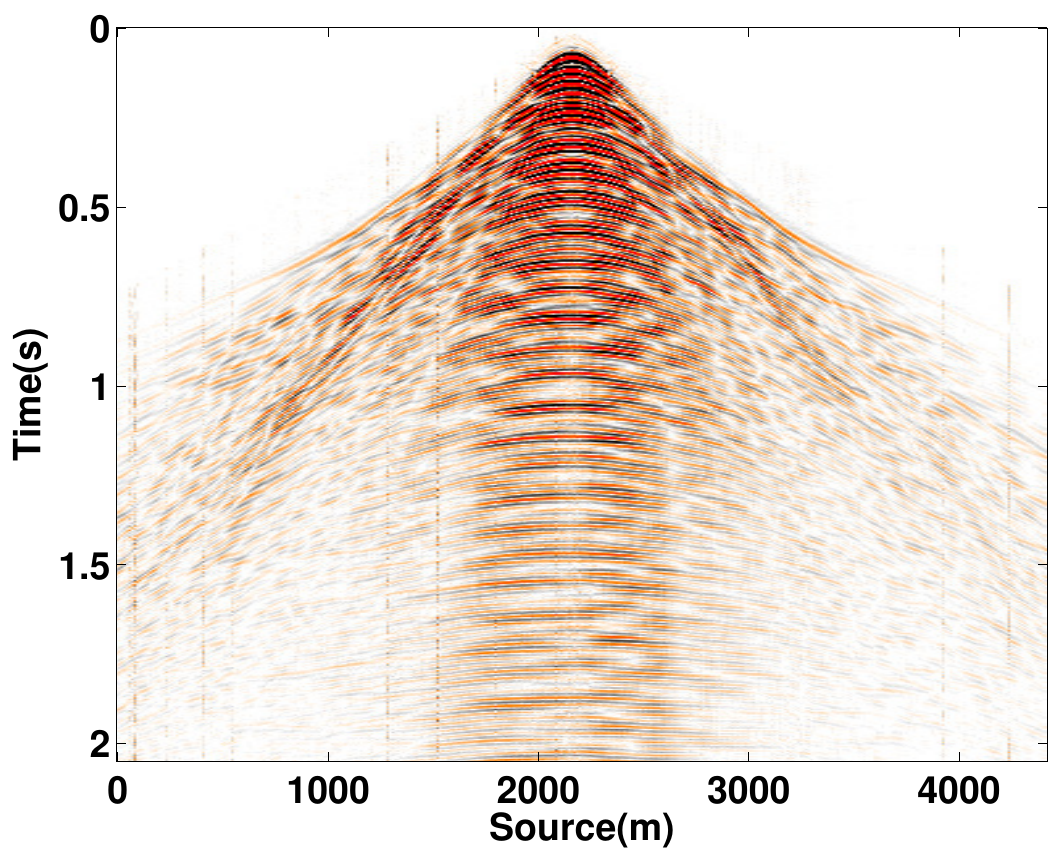}}
    \subfigure[]{\includegraphics[scale=0.35]{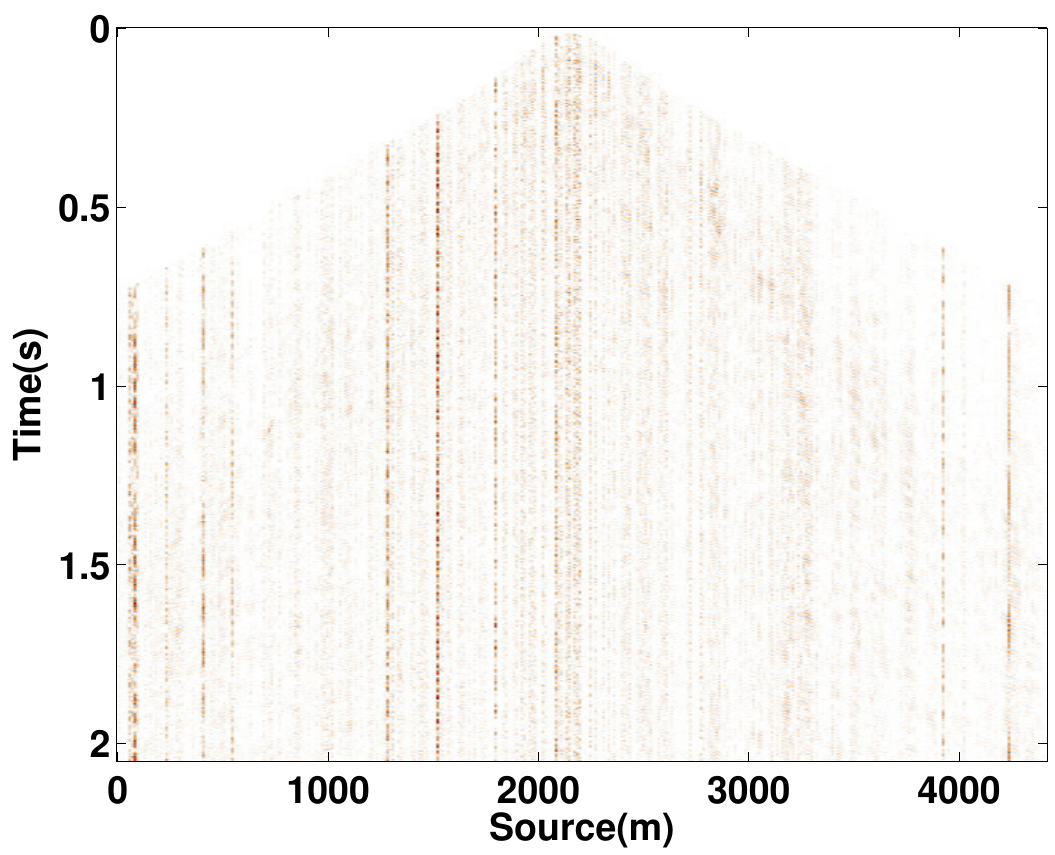}}
  \end{center}
  \caption{Comparison of interpolation and denoising results for the Student's t and least-squares misfit function.
  (a) $50\%$ subsampled common receiver gather with another 10 ${\%}$ of the shots replaced 
  by large errors.
  (b) Recovery result using the least-squares misfit function.
  (c,d) Recovery and residual results using the student's t misfit function with a SNR of 17.2 dB. }  
  \label{fig:9}
\end{figure}

\begin{figure}
  \begin{center}
    \subfigure[]{\includegraphics[scale=0.6]{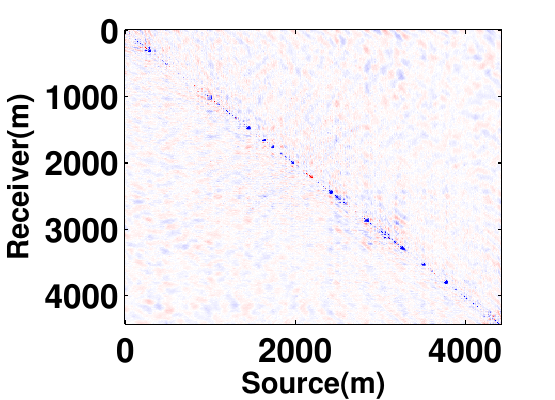}}
    \subfigure[]{\includegraphics[scale=0.6]{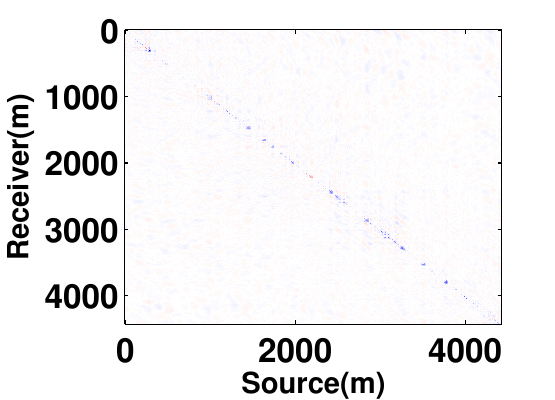}}
   \end{center}
  \caption{
 Residual error for recovery of 11 Hz slice
 (a)  without weighting and 
 (b) with weighting using true support. SNR in this case is improved by 1.5 dB. }
  \label{fig:10}
\end{figure}

\begin{figure}
  \begin{center}
    \subfigure[]{\includegraphics[scale=0.6]{./Fig1/orig11hznosup.pdf}}
    \subfigure[]{\includegraphics[scale=0.6]{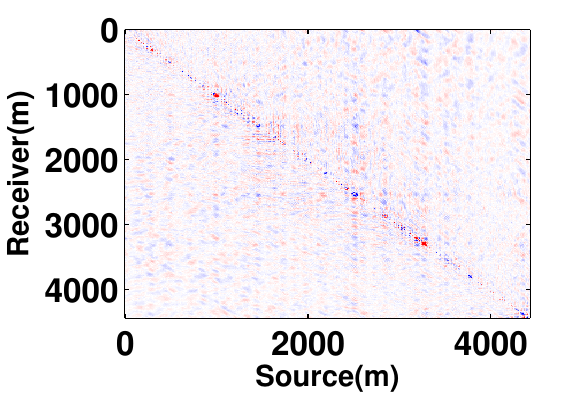}}\\
     \subfigure[]{\includegraphics[scale=0.6]{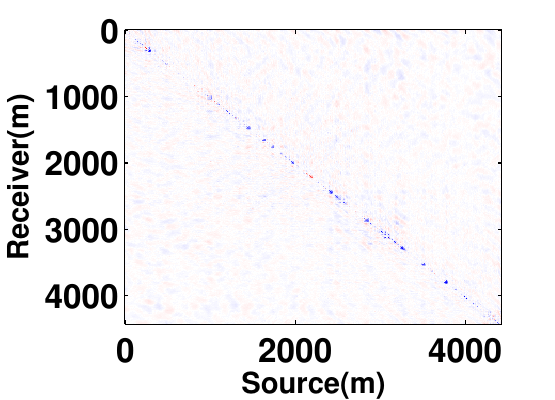}}
     \subfigure[]{\includegraphics[scale=0.6]{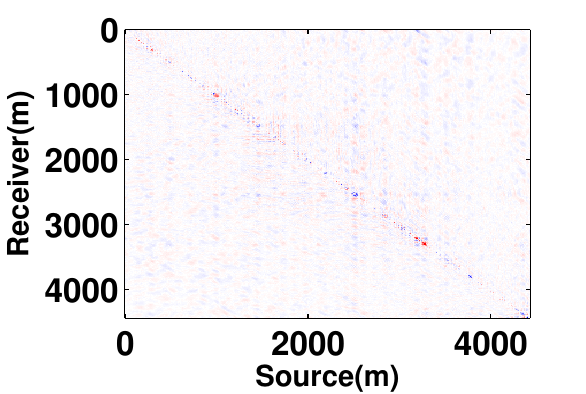}}
  \end{center}
  \caption{
 Residual of low frequency slice at 11 Hz (a) without weighing 
 (c) with support from 10.75 Hz frequency slice.
  SNR is improved by 0.6 dB. 
  Residual of low frequency slice at 16 Hz 
  (b) without weighing 
  (d)  with support from 15.75 Hz frequency slice. 
  SNR is improved by 1dB. Weighting using learned support
   is able to improve on the unweighted interpolation results.}
  \label{fig:11}
\end{figure}

\begin{figure}
  \begin{center}
    {\includegraphics[scale=0.8]{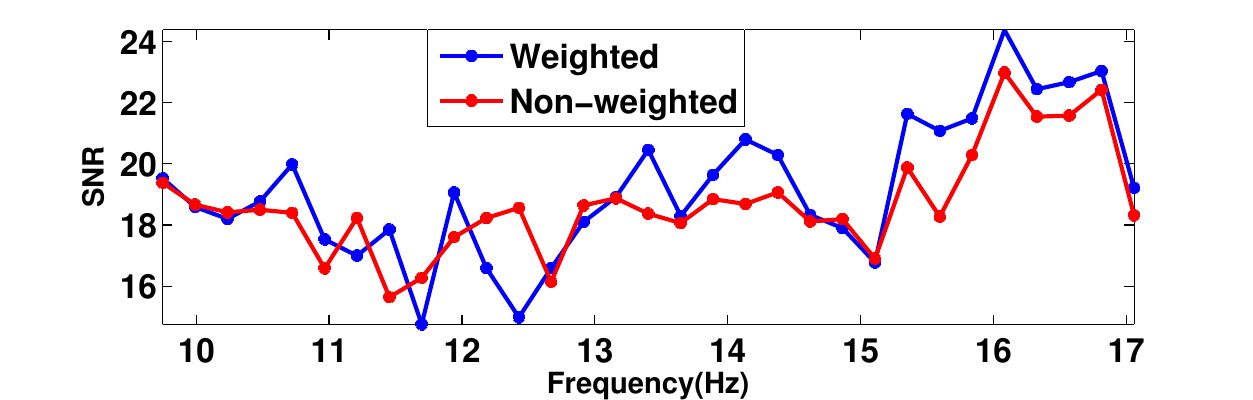}}
  \caption{Recovery results of practical scenario in case of weighted factorized formulation over a frequency range of 9-17 Hz. The weighted formulation outperforms the
  non-weighted for higher frequencies.
For some frequency slices, the performance of the non-weighted algorithm is better,
because the weighted algorithm can be negatively affected 
when the subspaces are less correlated. }  \label{fig:12}
  \end{center}
\end{figure}


%

\section{Conclusions}
We have presented a new method for matrix completion. 
Our method combines the Pareto curve approach 
for optimizing~\eqref{BPDN} formulations 
with SVD-free matrix factorization methods. 

We demonstrated the modeling advantages of the~\eqref{BPDN} formulation
on the Netflix Prize problem, and obtained high-quality reconstruction
results for the seismic trace interpolation problem. 
Comparison with state of the art methods for the~\eqref{BPDN}
formulation showed that the factorized formulation is faster 
than both TFOCS and classic SPG$\ell_1$ formulations that rely 
on the SVD. 
\black{The presented factorized approach also has 
a small memory imprint and does not rely on SVDs, 
which makes this method applicable to truly large-scale problems.}

We also proposed two extensions. First, using robust penalties $\rho$
in~\eqref{BPDN}, we showed that simultaneous interpolation and 
denoising can be achieved in the extreme data contamination case, 
where 10\% of the data was replaced by large outliers. 
Second, we proposed a weighted extension~\eqref{wBPDN},
and used it to incorporate subspace information we learned
on the fly to improve interpolation in adjacent frequencies. 

\section{Appendix}
\paragraph{Proof of Theorem~\ref{thm:genFact}}
Recall~\cite[Lemma 2.1]{Burer03localminima}: if $SS^T = KK^T$, then $S = KQ$ for some orthogonal matrix $Q \in \mathbb{R}^{r\times r}$. 
Next, note that the objective and constraints of~\eqref{generalXfact} are given in terms of $SS^T$, and for any orthogonal $Q \in \mathbb{R}^{r\times r}$, we have $SQQ^TS^T = SS^T$, so $\bar S$ is a local minimum of~\eqref{generalXfact} if and only if $\bar S Q$ is a local minimum for all orthogonal $Q \in \mathbb{R}^{r \times r}$.  

If $\bar Z $ is a local minimum of~\eqref{generalX}, then any factor $\bar S$ with $\bar Z = \bar S \bar S^T$ 
is a local minimum of~\eqref{generalXfact}. Otherwise, we can find a better solution $\tilde S$ in the neighborhood of $\bar S$, and 
then $\tilde Z := \tilde S \tilde S^T$ will be a feasible solution for~\eqref{generalX} in the neighborhood of $\bar Z$ (by continuity of the map $S \rightarrow SS^T$).
 
We prove the other direction by contrapositive. If $\bar Z$ is not a local minimum for~\eqref{generalX}, then you can find a sequence of feasible solutions $Z_k$ with $f(Z^k) < f(\bar Z)$ and $Z_k \rightarrow \bar Z$.  For each $k$, write $Z_k = S_k S_k^T$. 
Since $Z_k$ are all feasible for~\eqref{generalX}, so $S_k$ are feasible for~\eqref{generalXfact}. By assumption
$\{Z_k\}$ is bounded, and so is $S_k$; we can therefore find a subsequence of $S_j \rightarrow \tilde S$ with $\tilde S \tilde S^T  = \bar Z$,  
and $f(S_j S_j^T) < f(\tilde S \tilde S^T)$. In particular, we have $\bar Z = \bar S \bar S^T = \tilde S \tilde S^T$, and $\tilde S$ is not a local
minimum for~\eqref{generalXfact}, and therefore (by previous results) $\bar S$ cannot be either.

\clearpage
\bibliography{siam}
\bibliographystyle{siam}

\end{document}